\documentclass{article}

\usepackage{algorithmicx}
\usepackage{algorithm}
\usepackage{algpseudocode}
\usepackage{amsbsy}
\usepackage{amsfonts}
\usepackage{amsmath}
\usepackage{amssymb}
\usepackage{amsthm}
\usepackage{appendix}
\usepackage{arxiv}
\usepackage{balance}
\usepackage{booktabs}
\usepackage{caption}
\usepackage{doi}
\usepackage{graphicx}
\usepackage{epstopdf}

\usepackage{inputenc}
\usepackage{lipsum}
\usepackage{listings}
\usepackage{makecell}
\usepackage{mathtools}
\usepackage{microtype}
\usepackage{manyfoot}
\usepackage{mathrsfs}
\usepackage{moresize}
\usepackage{multirow}
\usepackage{natbib}
\usepackage{nicefrac}

\usepackage{pgfplots}
\pgfplotsset{compat=1.18} 

\usepackage{subcaption}
\usepackage{tabularx}
\usepackage{textcomp}
\usepackage{tikz}
\usepackage{url}
\usepackage{xcolor}

\usetikzlibrary{calc}
\usetikzlibrary{external}
\usetikzlibrary{fadings}
\usetikzlibrary{intersections}
\usetikzlibrary{pgfplots.fillbetween}
\usetikzlibrary{shadings}

\usepackage{hyperref}

\pgfdeclarelayer{pre main}

\newcommand{\N}{\mathbb{N}}
\newcommand{\Z}{\mathbb{Z}}
\newcommand{\R}{\mathbb{R}}

\newcommand{\E}{\mathbb{E}}
\newcommand{\bb}[1]{\boldsymbol{#1}}

\renewcommand{\shorttitle}{\textit{arXiv} Template}

\definecolor{CadetBlue}{rgb}{0.3,0.2,0.5}
\definecolor{Orange}{cmyk}{0.1,0.5,0.6,0}
\definecolor{Lila}{rgb}{0.5,0,0.5}
\definecolor{Head}{cmyk}{0.3,0.5, 0.2,0}
\definecolor{darkred}{rgb}{0,0.5,0}
\definecolor{TUMiddleGreen}{cmyk}{0.44,0,0.98,0.43}
\definecolor{TULightGreen}{cmyk}{0.4,0,0.89,0}
\definecolor{TUForestGreen}{cmyk}{0.94,0,0.59,0.64}
\definecolor{TULightYellow}{cmyk}{0,0.43,0.72,0}
\definecolor{ULMcyan}{rgb}{0.4,0.5,0.6}
\definecolor{TUGreen}{RGB}{12,94,82}
\definecolor{TUR}{RGB}{160,0,0}
\definecolor{teal}{RGB}{0, 128, 128}
\definecolor{wheat}{RGB}{245, 222, 179}
\definecolor{crimson}{RGB}{220, 20, 60}
\definecolor{purple}{RGB}{128, 0, 128}

\theoremstyle{thmstyleone}%
\newtheorem{lemma}{Lemma}%

\theoremstyle{thmstyletwo}%

\theoremstyle{thmstylethree}%
\newtheorem{definition}{Definition}%

\begin{document}

\title{Spatio-temporal probabilistic forecast using MMAF-guided learning}


\author{
    \hspace{1mm}Leonardo Bardi \\
   	Department of Mathematics\\
	TU Chemnitz\\
	Reichenhainer Straße 41\\
	\texttt{leonardo.bardi@math.tu-chemnitz.de} \\
    \And
    \hspace{1mm}Imma.~V. Curato \\
	Department of Mathematics\\
	TU Chemnitz\\
	Reichenhainer Straße 41\\
	\texttt{imma-valentina.curato@math.tu-chemnitz.de} \\
	\And
	\hspace{1mm}Lorenzo Proietti \\
	Department of Mathematics\\
	TU Chemnitz\\
	Reichenhainer Straße 41\\
	\texttt{lorenzo.proietti@math.tu-chemnitz.de} \\
	}

\date{}

\renewcommand{\headeright}{}
\renewcommand{\undertitle}{}
\renewcommand{\shorttitle}{Spatio-temporal probabilistic forecast using MMAF-guided learning}

\hypersetup{
pdftitle=[Spatio-temporal probabilistic forecast using MMAF-guided learning]{Spatio-temporal probabilistic forecast using MMAF-guided learning},
pdfauthor={Leonardo Bardi, Lorenzo Proietti, Imma V.~Curato},
pdfkeywords={spatio-temporal data, feed-forward neural networks, generalized Bayesian learning, ensemble forecast, calibration},
}

\maketitle

\begin{abstract}
    We present a theory-guided generalized Bayesian methodology for spatio-temporal raster data, which we use to train an ensemble of stochastic feed-forward neural networks with Gaussian-distributed weights. The methodology incorporates the dependence and causal structure of a spatio-temporal Ornstein-Uhlenbeck process into training and inference by enforcing constraints on the design of the data embedding and the related optimization routine. In inference mode, the networks are employed to generate causal ensemble forecasts by applying different initial conditions at different horizons. We call this workflow MMAF-guided learning. Experiments conducted on both synthetic and real data demonstrate that our forecasts remain calibrated across multiple time horizons. Moreover, we show that on such data, shallow feed-forward architectures can achieve performance comparable to, and in some cases better than, convolutional or diffusion deep learning architectures used in probabilistic forecasting tasks.
\end{abstract}




\maketitle

\section{Introduction}\label{sec1}
Raster datasets are (2D) regularly sampled spatio-temporal data that are nowadays generated in environmental monitoring, from satellite observations, in finance, neuroscience, and brain imaging \citep{raster3,raster2, brain}. Such data are typically characterized by an unknown spatio-temporal dependence structure and are generated by a complex dynamical system. Assessing forecast uncertainty for such data, which can arise from noisy measurements or intrinsic stochasticity and the framework used to model them, is of utmost importance for decision-making. We propose in this paper a so-called \emph{theory-guided methodology} to address this issue, which is a hybrid approach that combines a data-generating process with a deep learning model. One of the most successful applications of the latter is called \emph{physics-informed machine learning}, which embeds physical laws into the training of a deep learning model, enabling accurate forecasts with limited data as discussed in \cite{theoryg, reviewPIM, PI, Intrinsic, Nature}. In these works, the data are typically treated as samples from solutions of partial differential equations (PDEs), and a forecast task is equivalent to approximating the PDE solution field at a given spatio-temporal position. If we further assume that the data are noisy (while maintaining deterministic generation), we need to assume that the noise is additive, independent, and Gaussian distributed. Under such assumptions, methods such as B-PINN \citep{BPINN} or Gaussian process regression \citep{Gpde} can then quantify aleatoric uncertainty using a physics-informed approach. To the best of our knowledge, setups in which data generation is assumed to be driven by a random field have been investigated only for Gaussian processes, as discussed, for example, in \cite{GPPIM, VGP}. 

Raster datasets are typically considered as generated by a random field, given the intrinsic stochasticity of most phenomena they measure. However, the Gaussian assumption is often unrealistic for them and limits the ability to assess a reasonable approximation to the \emph{(conditional) predictive distribution of the data} for generating forecasts. In this paper, we extend a workflow called MMAF-guided learning, first introduced in \cite{mmafGL} for linear models and Gibbs posterior distributions, to stochastic feed-forward neural networks with Gaussian-distributed weights. Modeling the predictive distribution and having a routine that can sample from it (or its approximation) is what we call throughout a \emph{probabilistic forecasting task}. Moreover, in our framework, uncertainty about a forecast refers to a combination of the aleatoric and epistemic uncertainties, the latter referring to the uncertainty in the neural network's parameters.

In the literature, rather than using a theory-guided methodology, a \emph{pure data-driven generative model} can also be directly employed to map observed raster data to an approximation of the predictive distribution. Several state-of-the-art deep learning architectures can be used for this aim, such as ConvLSTM \citep{XZH15}, ConvGRU \citep{TLX20}, DiffSTG \citep{HEY25}, and many others, see for a review \cite{AIWeather, WindSpeed, reviewD}. Forecasting spatio-temporal datasets, however, entails several limitations when using only deep learning architectures, primarily due to the difficulty of fully representing interactions between temporal and spatial dimensions, thereby limiting model performance. For example, accurately capturing spatio-temporal correlation in the data is crucial for enhancing performance in spatio-temporal probabilistic forecasting. On the other hand, focusing solely on this point limits the models' ability to capture causal relationships in space and time among the random variables that generate the observed dataset. In fact, the events that result in records in our data are influenced by neighboring locations and past observations, and exhibit an undeniable causal structure that is often not considered by cutting-edge generative algorithms applied to spatio-temporal data. Different approaches are currently being explored in the generative AI community \citep{Causal_diff}. However, typically, if the model we design can learn causal relationships in the data, it is not well-suited to probabilistic forecasts over multiple time horizons. Moreover, all such networks are computationally very costly. Resorting to the workflow introduced in the paper also has the target of overcoming such bottlenecks by reducing the computational cost of training a generative architecture, and at the same time modeling the dependence and the causal structure of the data. 

In MMAF-guided learning, we assume that a mixed moving average field $(\bb{Z}_t(x))_{(t,x) \in \R \times \R}$ (MMAF, in short) generates an observed raster dataset. Such a model setup allows modeling the dependence structure of the data without specifying their distribution and incorporates a causal structure represented by a so-called cone-shaped \emph{ambit set} \citep{Ambit}. This definition of causality is inspired by that of a lightcone in special relativity and, as a causal model, falls within the interpretation given by the potential outcome framework discussed by \cite{Rubin}. Moreover, since the index set lies in $\R \times \R$, MMAF can be used as a data-generating process independently of the time and spatial scale of our data. Such random fields have been so far employed to model data in environmental monitoring \citep{HIG02, STOU}, finance \citep{CS19}, imaging analysis \citep{brain}, and electricity networks \citep{Graph22} and belong to the broader class of the Ambit processes \citep{Ambit}. The latter represent a general framework for modeling spatio-temporal phenomena that also encompasses certain solutions of stochastic partial differential equations as discussed in \cite{spde}.
We focus in this work on a specific field in this class called the spatio-temporal Ornstein-Uhlenbeck process (in short, STOU), which has been thoroughly described in \cite{STOU}. This has the effect of focusing on data with exponentially decaying autocorrelation functions along the spatial and temporal dimensions. Other possible model specifications are allowed in MMAF-guided learning, but they are outside the scope of the present paper. We refer the reader to \cite{mmafGL, Leo} for further details on the possible applications of the methodology to data with a power decaying autocorrelation function (in the case of short and long memory) and to Section \ref{sec5} for further considerations on non-stationary setups.

The causal structure of the STOU process is used to define an embedding of the spatio-temporal index space, which allows us to define features that inherit the field's dependence structure, expressed using $\theta$-lex coefficients, a measure of asymptotic dependence introduced in \cite{CSS20}, and to provide a causal interpretation for the forecasts MMAF-guided learning can produce. The result of this procedure is obtaining a low-dimensional feature representation of the initial data by using standard statistical techniques, such as the method of moments. The features we create are used to train an ensemble of stochastic feed-forward neural networks (one for each spatial position) with Gaussian-distributed weights. The training is performed using a generalized Bayesian optimization routine \citep{OptCentric} based on a PAC Bayesian bound for spatio-temporal $\theta$-lex weakly dependent data. The latter has been proven and discussed in \cite{mmafGL} and unlike the classical ones available in the literature for $\alpha$-mixing or $\theta_{1,\infty}$-dependent processes, presented for example in \cite{Hostile} and \cite{AW12}, can be estimated from observed data. In \cite{mmafGL}, the PAC Bayesian bound is not directly used in the design of a training routine, but to study the generalization performance of the Gibbs posterior distribution on a linear model. It is also important to highlight that using a generalized Bayesian methodology means we do not need to model the likelihood of the data or choose a prior for the network parameters. Instead, we just need to choose a \emph{reference distribution} that serves as a regularizer in the parameter space \citep{C2004}. A so-called \emph{generalized posterior distribution} is then directly learned by solving the optimization problem at the base of the training routine. Therefore, our methodology differs from standard Bayesian ones used to analyze spatio-temporal data, as hierarchical Bayes \citep{CW11, Monte}, simply because it is not based on the use of the Bayes' theorem.

Finally, the trained networks are used to generate an \emph{ensemble forecast}. Our ensembles are determined using new input data (not belonging to the training dataset) for each future time horizon we aim to forecast. Our approach to forecasting aligns with what is done, for example, when using weather forecast routines, which, at a specific time interval, use data assimilation of observational data to update the current state of a numerical weather forecast model and produce ensemble forecasts \citep{DataAss, AIWeather}. Such forecasting schemes are suited for data sets that are updated in real time, but whose training phase is very expensive to repeat at each new evaluation time. The calibration of the ensemble forecast is tested employing the continuous ranked probability score (CRPS) \citep{gneiting05, gneiting07}. With respect to the generalized Bayesian literature, we use the term calibration in this paper in a different sense than in \cite{CalibBay} or \cite{Knob24}. Here, the authors analyze the calibration of generalized posterior distributions on the parameter space, whereas we analyze whether the outputs of stochastic feed-forward neural networks with parameters sampled from such distributions are calibrated. 

The key contributions of the paper are the following.
\begin{itemize}
\item We present a novel workflow called MMAF-guided learning, which is a theory-guided machine learning procedure that assumes a (2D) raster dataset is generated by an STOU process and uses stochastic feed-forward neural network to solve a probabilistic forecasting task.
We use this assumption to define a feature-extraction routine at a given spatial position, as well as the training and inference procedure for the networks.
\item We define a training routine based on a constrained minimization of a PAC Bayesian bound for spatio-temporal $\theta$-lex weakly dependent data. The constraint in the optimization problem forces an exponentially decaying autocorrelation function in the modeling, whereas the features employed enforce a causal model in the learning based on cone-shaped ambit sets. As a result of the optimization problem, we determine a Gaussian generalized posterior distribution over the network parameters.
\item Our workflow is designed for an ensemble of stochastic feed-forward neural networks (one for each spatial position of interest) and also includes validating the \emph{reference distribution} using the continuous ranked probability score (CRPS). The ensemble forecasts obtained in inference mode have a causal interpretation.
\item  We analyze the performance of the generalized Bayesian routine defined in the paper using synthetic and real data sets. For each considered data set, we compare the training costs and the probabilistic forecasts of stochastic feed-forward neural networks with those from the following state-of-the-art generative models: ConvLSTM, ConvGRU, and DiffSTG. In our empirical study, we observe that shallow feed-forward neural networks achieve comparable forecast performance to that of deep learning architectures. In general, our proposed training routine results in lower training computational costs and uses less data.
\end{itemize}

The paper is organized as follows. In \textbf{Section \ref{sec2}}, we describe the modeling assumptions that guide the design of the embedding and the feature extraction routine. In \textbf{Section \ref{new}}, we present the optimization problem inspired by a PAC Bayesian bound for $\theta$-lex weakly dependent data, and the specific design of the training routine for stochastic feed-forward neural networks. We also discuss the inference mode related to the latter and the causal interpretation of the ensemble forecasts. The experimental setup is described in \textbf{Section \ref{sec3}}, whereas the comparison between the performance of stochastic feed-forward neural networks and generative models is presented in \textbf{Section \ref{sec4}}. \textbf{Section \ref{sec5}} concludes.

\section{Spatio-temporal embedding and features}\label{sec2}
In this section, we present the theoretical background and the algorithm for defining a spatio-temporal embedding, along with a related feature extraction routine for 2D raster datasets. The latter is the first step of the \emph{MMAF-guided learning workflow}. We indicate random quantities throughout with bold symbols.

\subsection{Modeling Assumptions}
\label{modass}
A 2D raster data set is an observed data set  $(\tilde{Z}_t(x))_{(t,x)\in  \mathbb{T} \times \mathbb{L}}$ on a regular spatial grid $\mathbb{L} \subset \mathbb{R}$ across times $\mathbb{T}=\{t_0+h_t i: i=1,\ldots,N \}$, where $t_0 \in \R$ represents the initial observation time point and $h_t \in \R$ represents the sample frequency of the data. Moreover, we consider throughout the following decomposition, that is 
\begin{equation}
	\label{decomposition}
	\tilde{Z}_t(x)= \mu_t(x)+ Z_t(x)
\end{equation}
holds, and if measurement errors are present are indistinguishable from the stochastic signal present in the data.
Here, $\mu_t(x)$ is a deterministic function, and $Z_t(x)$ is considered a realization from a zero-mean stationary spatio-temporal Ornstein-Uhlenbeck process \citep{STOU}, STOU in short, defined as 
\begin{equation}
		\label{stou}
		\bb{Z}_t(x):=\int_{A_t(x)} \exp(-A (t-s)) \Lambda(ds,d\xi),\,\, (t,x) \in \R \times \R
	\end{equation}
where $A >0$ is called \emph{mean reverting parameter}, and $A_t(x)$ is defined as 
	\begin{equation}
		\label{lightcone}
		A_t(x) := \big\{(s,\xi)\in\mathbb{R}\times\mathbb{R}: s\leq t \,\, \textit{and}\,\, \|x-\xi\| \leq c |t-s|  \big\}.
	\end{equation}
 $\Lambda$ is a  L\'evy basis. More in detail, let $\mathcal{B}_b(\R \times \R)$ be the Borel $\sigma$-algebra containing all its Lebesgue bounded sets. 

\begin{definition}\label{def:levybasis}
A family of $\R$-valued random variables $\Lambda=\{\Lambda(B): B \in \mathcal{B}_b(\R \times \R)\}$ is called a L\'evy basis on $(\R \times \R,\mathcal{B}_b(\R \times \R))$ if it is an independently scattered and infinitely divisible random measure. This means that:
\begin{itemize}
		\item[(i)]  For a sequence of pairwise disjoint elements of $\mathcal{B}_b(\R \times \R)$, say $\{B_i, i \in \N \}:$
		\begin{itemize}
			\item $\Lambda (\bigcup_{i\in\N} B_i)= \sum_{i\in\N}\Lambda(B_i)$ almost surely when $\bigcup_{i\in\N} B_i \in \mathcal{B}_b(I)$
			\item and $\Lambda(B_i)$ and $\Lambda(B_j)$ are independent for $i\neq j$.
		\end{itemize}
		\item[(ii)] Let $B \in \mathcal{B}_b(\R \times \R)$. Then, the random variable $\Lambda(B)$ is infinitely divisible, i.e., for any $i \in \N$, there exists a law $\mu_i$ such that the law $\mu_{\Lambda(B)}$ can be expressed as $\mu_{\Lambda(B)}=\mu_i^{*i}$, the $i$-fold convolution of $\mu_i$ with itself.
	\end{itemize}
\end{definition}

For more details on infinitely divisible distributions, we refer the reader to \cite{S}. In the following, we will restrict ourselves to L\'evy bases which are homogeneous in space and time and factorizable. As apparent from (\ref{stou}) and Definition \ref{def:levybasis}, an STOU process has an integral definition with a possible non-Gaussian noise structure, and it is general enough to model the \emph{randomness} contained in observed raster data. 

Another special characterization of the STOU process is its \emph{spatio-temporal causal structure}, which is embodied by the set $A_t(x)$ formally called an \emph{ambit set}, see \cite{Ambit}. The latter satisfies the following properties.

\begin{definition}
    A family of ambit sets $(A_t(x))_{(t,x)\in\R\times\R}$ is a collection of subsets $\R\times\R$ which satisfies the following properties:
    \begin{enumerate}
        \item translation invariance, i.e. $A_t(x)=A_0(0)+(t,x)$ for all $(t,x)\in\R\times\R$;
        \item $A_s(x)\subset A_t(x)$ for all $s<t$ and for all $x\in\R$;
        \item non-anticipative, i.e. $A_t(x)\cap(t,+\infty)\times\R=\emptyset$
    \end{enumerate}
\end{definition}

To explain why a cone-shaped ambit set as defined in (\ref{lightcone}) introduces a causal structure in space-time, we borrow the interpretation given to a \emph{lightcone} in special relativity. In fact, as a lightcone describes the possible paths that the light can make in space-time, leading to a space-time point $(t,x)$ and the ones that lie in its future, we use its geometry (in the Euclidean space $\R \times \R$)  to identify the space-time points having a causal relationship in the past and the future. The set $A_t(x)$ corresponds to what we call a \emph{past cone of influence}, whereas the set 
\begin{equation}
	\label{future_lightcone}
	A_t(x)^+=\{(s,\xi)\in \R\times\R : s >t \, \textrm{and} \, \,  \|x-\xi\| \leq c |t-s| \},
\end{equation}
is the \emph{future cone of influence} related to the spatial-time point $(t,x) \in \R \times \R$. By using an STOU process as the underlying model, we implicitly assume that the following sets
\begin{equation}
	\label{lightcone2}
	l(t,x)=\{\bb{Z}_s(\xi): (s,\xi) \in A_t(x) \setminus (t,x) \} \,\,\, \textrm{and} \,\,\, l^{+}(t,x)=\{\bb{Z}_s(\xi): (s,\xi) \in A_t(x)^+ \}
\end{equation}
are respectively describing the values of the field that have a direct influence on the determination of $\bb{Z}_t(x)$ and the future field values influenced by $\bb{Z}_t(x)$. We can then infer the causal relationship between space-time points described above by estimating the constant $c$ from the observed raster data, which we call \emph{the speed of information propagation} in the physical system under analysis. The concept of causality used in the paper can be inscribed in the potential outcomes framework described by \cite{Rubin}.

The STOU is a stationary and Markovian random field. Moreover, an STOU exhibits exponentially decaying temporal autocorrelation (just like the temporal Ornstein-Uhlenbeck process) and spatial autocorrelation. In addition, this class of fields admits non-separable autocovariances, which are desirable in practice \citep{STOU}. The STOU is a member of a broader class of random fields called \emph{mixed moving average fields}, MMAF in short, see \cite{mmafGL}. All MMAFs are stationary and $\theta$-lex-weakly dependent random fields. This means that independent of their distribution, we can assess the asymptotic dependence structure of such fields, described in our framework by the definition below.

\begin{definition}
    \label{thetadependence}
    Let $\textbf{Z}$ be an $\mathbb{R}$-valued random field. Then, $\textbf{Z}$ is called $\theta$-lex-weakly dependent if 
    $$
    \theta_{lex}(r)  = \sup_{u,v\in \mathbb{N}} \theta_{u,v}(r) \underset{r\to \infty}{\longrightarrow} 0
    $$
    where 
    \begin{align*}
		\theta_{u,v}(r) = \sup\left\{ \frac{|Cov(F(\textbf{Z}_{\Gamma}) , G(\textbf{Z}_{\Gamma^{\prime}})  )|}{\Vert F \rVert_{\infty} v Lip(G)} , F\in \mathcal{G}^*_u, \, G\in \mathcal{G}_v, \Gamma,\Gamma^\prime, |\Gamma|=u,\,|\Gamma^\prime|=v \right\}
    \end{align*}

    where $\mathcal{G}^*_u$ represents the set of the bounded functions on $\R^u$, $\mathcal{G}_v$ the set of bounded and Lipschitz functions on $\R^v$,  $\Gamma, \Gamma^\prime \subset \R \times \R$ sets of space-time points in lexicographic order such that $dist(\Gamma,\Gamma^{\prime}):= \inf_{i \in \Gamma, j \in \Gamma^{\prime}} \|i-j\|_{\infty}=r$, and where we call the related marginals of the field $\bb{Z}_{\Gamma}$ and $\bb{Z}_{\Gamma^{\prime}}$. We indicate with the sequence $(\theta_{lex}(r))_{r\in\mathbb{R}^+}$ the $\theta$-lex coefficients.
\end{definition}

For discrete samples of MMAF, it can be proven that $\theta$-lex weak dependence is a more general notion of dependence than $\alpha$ and $\phi$-mixing for spatio-temporal random fields, see  Section 2.3 in \cite{CSS20}. Moreover, knowing that the field $\bb{Z}_t(x)$ is $\theta$-lex weakly dependent it can be readily proven that the sample $(\bb{Z}_t(x))_{(t,x)\in \mathbb{T}\times \mathbb{L}}$, generating the observed (zero mean) raster data set, has the same dependence structure. 

In the case of the STOU process $(Z_{t}(x))_{(t,x) \in \R \times \R}$, assuming that $\E \left[ \lvert\bb{Z}_t(x)\rvert^2 \right] < \infty$, we have that

\begin{align}
\label{boundT}
\theta_{lex}(r) \leq  \Big( \,\frac{c}{A^2} Var(\Lambda^{\prime})  \,\exp\Big( -\,\underset{2\lambda}{\underbrace{\frac{A \min(2,c)}{c}}}\,r \Big) \Big)^{\frac{1}{2}}
\end{align}

where we call the parameter $\lambda$ the decay rate of the $\theta$-lex coefficients, see Appendix \ref{est} for detailed calculations.

In short, by choosing an STOU process as the model underlying the data generation, we impose a given causal structure in the data in function of the cone-shaped ambit sets and a dependence structure described by exponentially decaying $\theta$-lex coefficients. We assume throughout that the STOU process admits finite second moments.

We refer to Appendix \ref{est} for a brief review of the estimation methodology based on the method of moments necessary to estimate the parameter $A$, $c$, and the decay rate $\lambda$. 

\subsection{Features Extraction Routine}
\label{emb}
We design spatio-temporal features as input–output pairs, which are then partitioned into training, validation, and test sets from the 2D observed raster dataset using an embedding of the index set $\mathbb{T} \times \mathbb{L}$. W.l.o.g., we will assume that $h_t=1$ throughout. 

For a given spatio-temporal position $(t,x^*)$, we define related features by means of the index set 
\begin{equation}
\label{index}
I(t,x^*) = \{(i_s,x_s) \in \mathbb{T} \times \mathbb{L}: | x^* - x_s | \leq c(t-i_s), 0<t-i_s\leq p\}
\end{equation}
having cardinality $D:=|I(t,x^*)|$. Such a set corresponds to a discretization in space and time of the set $A_t(x)$ defined in (\ref{lightcone}). The entire number of examples employed in training, validation, and test set represents the \emph{features} we can automatically learn once an embedding is defined. In general, the parameter $c$ in the definition of the embedding is estimated from the data, see Appendix \ref{est}, $p$ is a hyperparameter, and $a$ is chosen following the rule (\ref{choice}) which is discussed in Section \ref{bound}. We then define the input-output pair $(X_i, Y_i)$ defined through the index sets $I(t_0+ia,x^*)$ for $i=1,\ldots,\lfloor \frac{N}{a} \rfloor$, as
\begin{align}
\label{inputoutput}
	X_i = (Z_{i_1}(x_1),\ldots,Z_{i_{D}}(x_{D}))^\top \,\,\, Y_i = Z_{t_0+ia}(x^*),
\end{align}
where $(i_s,x_s) \in I(t_0+ia,x^*)$ and $(i_s,x_s)<_{lex} (i_{s+1},x_{s+1})$ for all $s=1,\ldots, D-1$. The symbol $<_{lex}$ indicates that the point $(i_s,x_s)$ is in lexicographic order less than $(i_{s+1},x_{s+1})$ in $\R \times \R$, and the parameters $a >0$ and $p>0$ are chosen such that $a \geq p+1$. Note that each input vector $X_i$ corresponds to a sample from $l(t_0+ia,x^*)$ defined in (\ref{lightcone2}). Figure \ref{frame} represents an example in which the input-output pairs are represented: in this instance, the input vector corresponds to $3$ spatial locations. 

We call $S_m=\{(X_i, Y_i)\}_{i=1}^m$ the training data set for a given spatial position $x^*$ defined through the set $I(t_0+ia,x^*)$ for $i=1,\ldots,m$. Similarly, the validation set is composed of the input-output pair $(X_{m+1},Y_{m+1})$, which is defined using the embedding $I(t_0+(m+1)a,x^*)$, and the test set is composed by the input-output pairs $(X_{i},Y_{i})$ which is obtain from the sets $I(t_0+ia,x^*)$, for $i=m+2,\ldots, \lfloor \frac{N}{a} \rfloor$.  

We provide in Algorithm \ref{algorithm1} a detailed description of the feature extraction routine. Note that Algorithm \ref{algorithm1} varies if we choose a different data-generating process or if the dimension of the spatial dimension increases. In \cite{Leo}, for example, a feature extraction routine for 3D raster data sets under the assumption that the data are generated by a \emph{mixed spatio-temporal Ornstein Uhlenbeck process}, defined in \cite{mstou}, can be found. 

\pgfsetlayers{pre main,main}
\begin{figure}  
    \centering
    \caption{Feature extraction with parameters $c=1$, $p=1$, and $a=3$.} 
    \label{frame}
    \begin{tikzpicture}[
        frame/.style={
            yslant=1.7,
            inner sep=0pt,
            outer sep=0pt
        },     
        2d0/.style={
            frame,
            draw=gray,
            line width=0.5pt,
            opacity=.9,
            execute at begin node={
                \includegraphics[width=.4cm,height=4cm]{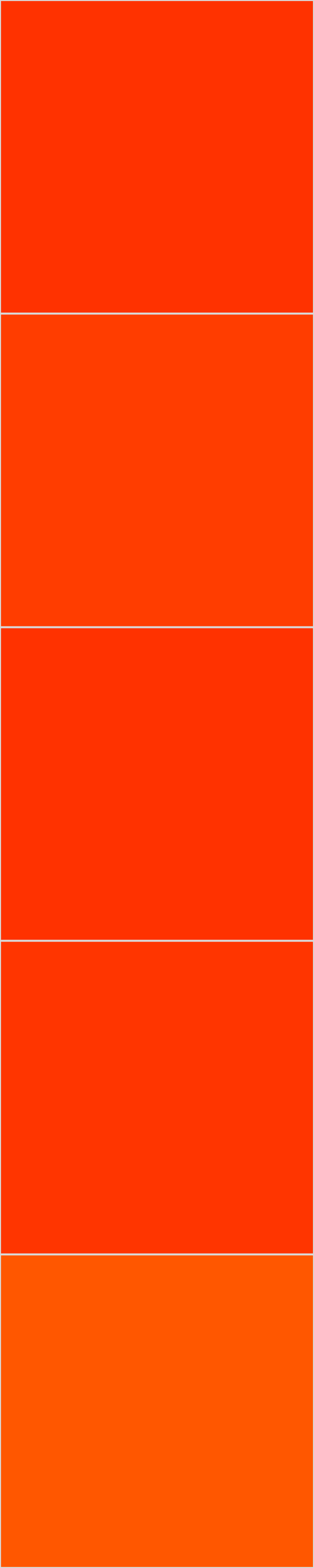}
            }
        },
        2d1/.style={
            frame,
            draw=gray,
            line width=0.5pt,
            opacity=.9,
            execute at begin node={
                \includegraphics[width=.4cm,height=4cm]{raster1.png}
            }
        },
        2d2/.style={
            frame,
            draw=gray,
            line width=0.5pt,
            opacity=.9,
            execute at begin node={
                \includegraphics[width=.4cm,height=4cm]{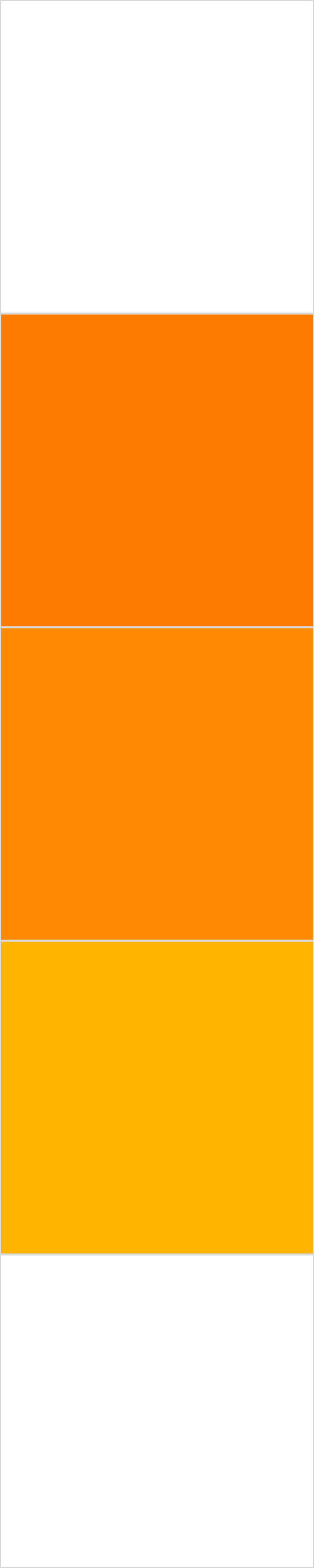}
            }
        },
        2d3/.style={
            frame,
            draw=gray,
            line width=0.5pt,
            opacity=.9,
            execute at begin node={
                \includegraphics[width=.4cm,height=4cm]{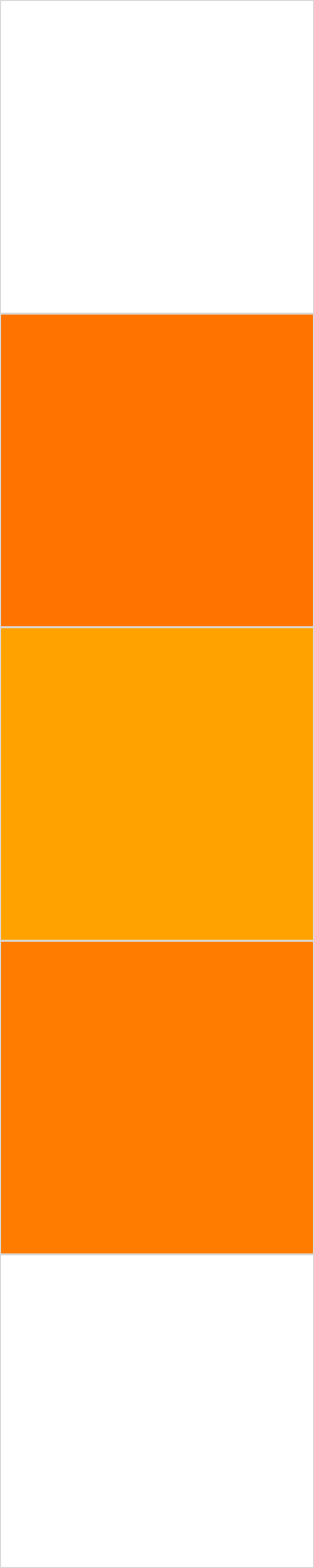}
            }
        },
        2d4/.style={
            frame,
            draw=gray,
            line width=0.5pt,
            opacity=.9,
            execute at begin node={
                \includegraphics[width=.4cm,height=4cm]{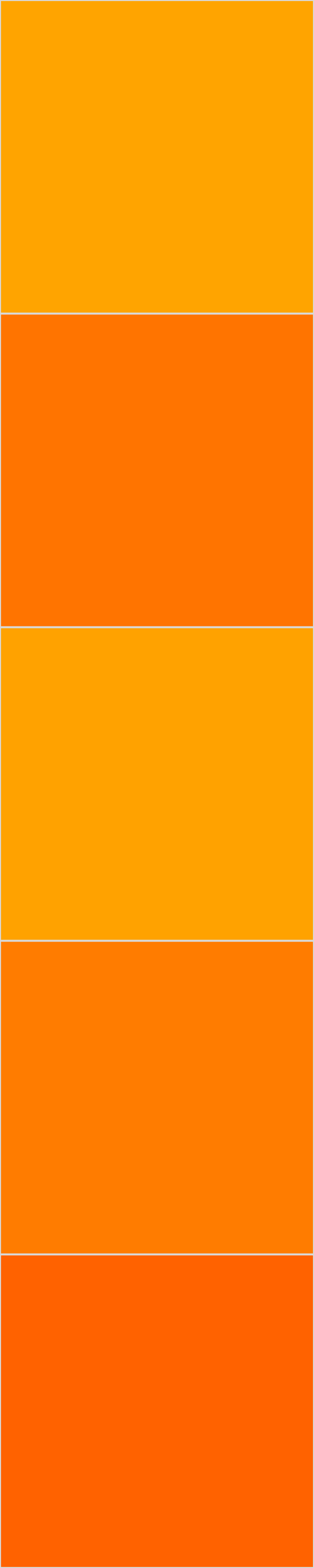}
            }
        },
        2d5/.style={
            frame,
            draw=gray,
            line width=0.5pt,
            opacity=.9,
            execute at begin node={
                \includegraphics[width=.4cm,height=4cm]{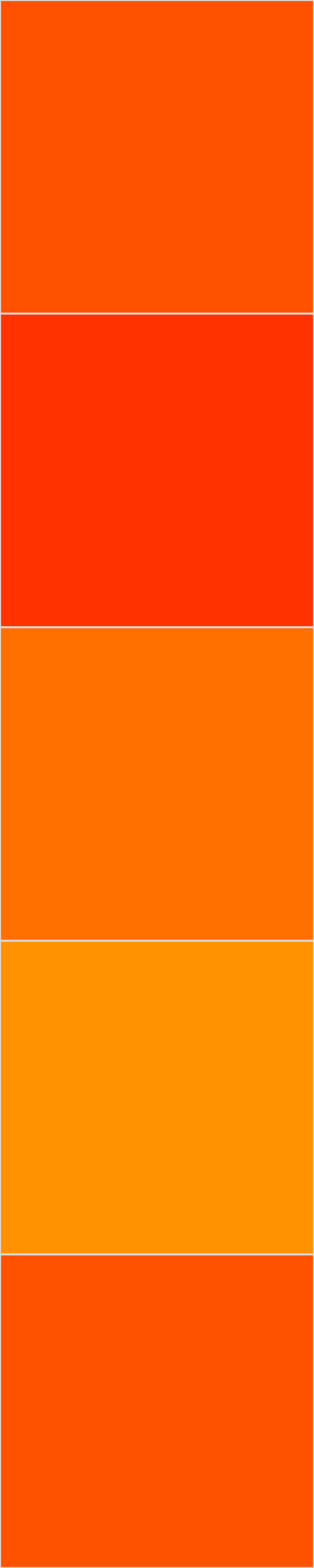}
            }
        },
        2d6/.style={
            frame,
            draw=gray,
            line width=0.5pt,
            opacity=.9,
            execute at begin node={
                \includegraphics[width=.4cm,height=4cm]{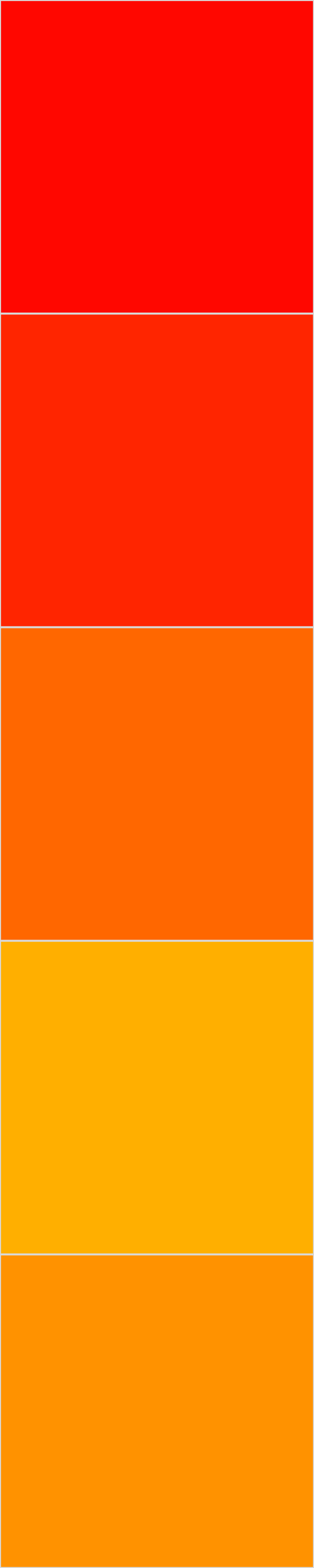}
            }
        },
        2d7/.style={
            frame,
            draw=gray,
            line width=0.5pt,
            opacity=.9,
            execute at begin node={
                \includegraphics[width=.4cm,height=4cm]{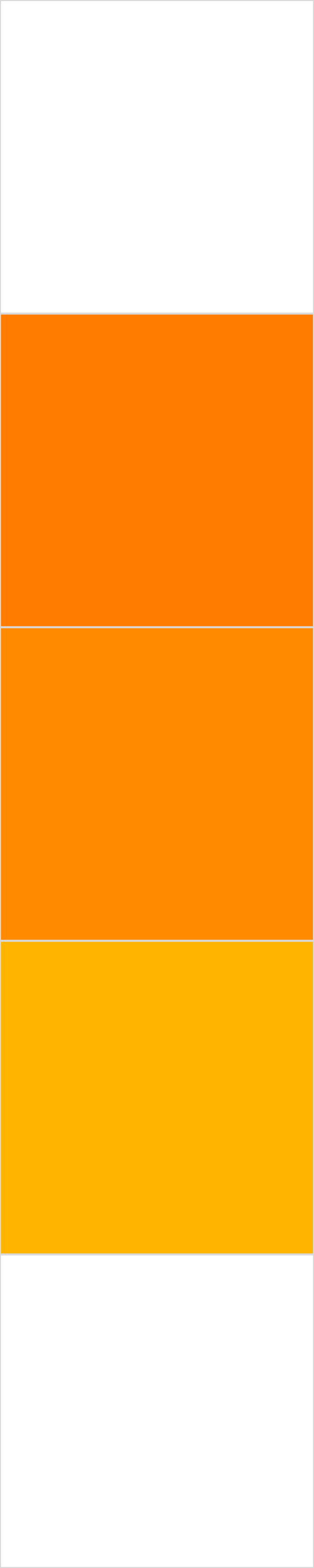}
            }
        },
        2d8/.style={
            frame,
            draw=gray,
            line width=0.5pt,
            opacity=.9,
            execute at begin node={
                \includegraphics[width=.4cm,height=4cm]{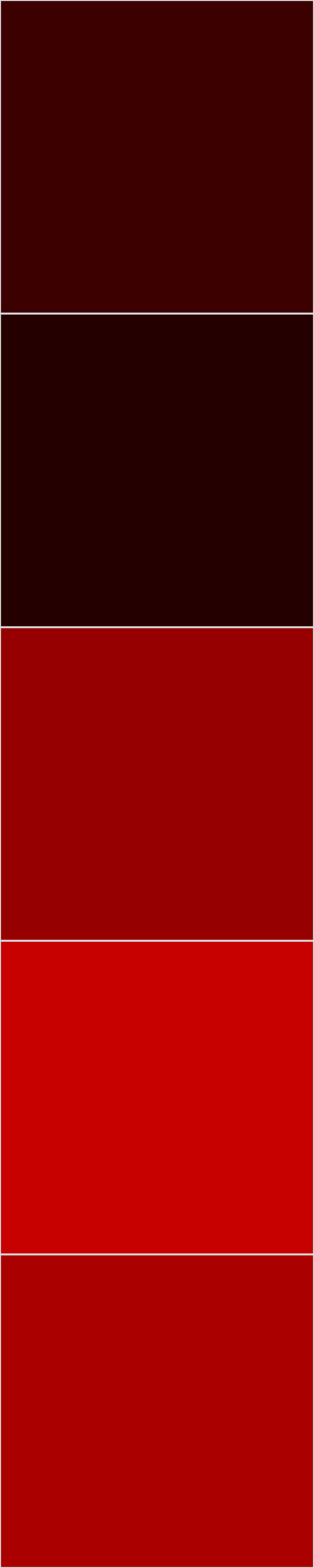}
            }
        },
        2t0/.style={
            frame,
            draw=gray,
            line width=0.8pt,
            opacity=.9,
            execute at begin node={
                \includegraphics[width=.4cm,height=4cm]{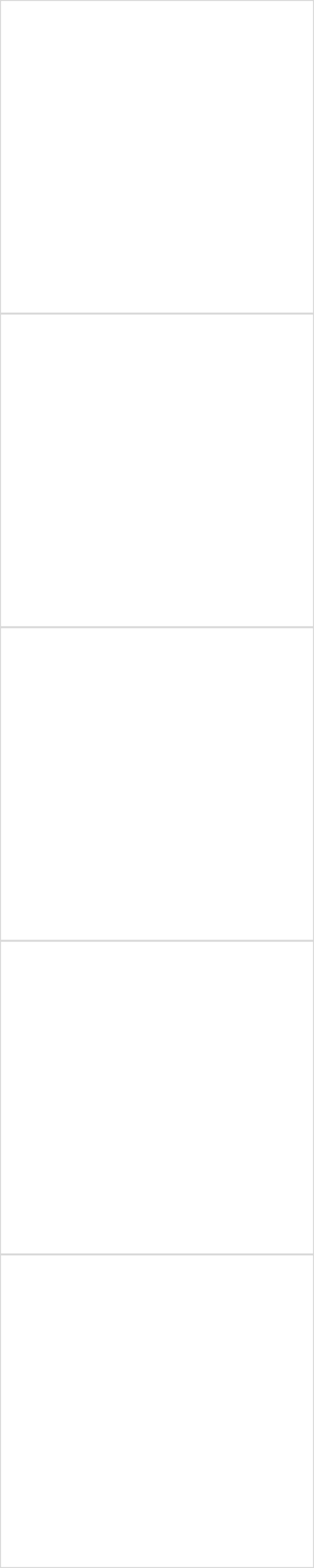}
            }
        },
        2t1/.style={
            frame,
            draw=crimson,
            line width=0.8pt,
            opacity=.9,
            execute at begin node={
                \includegraphics[width=.4cm,height=4cm]{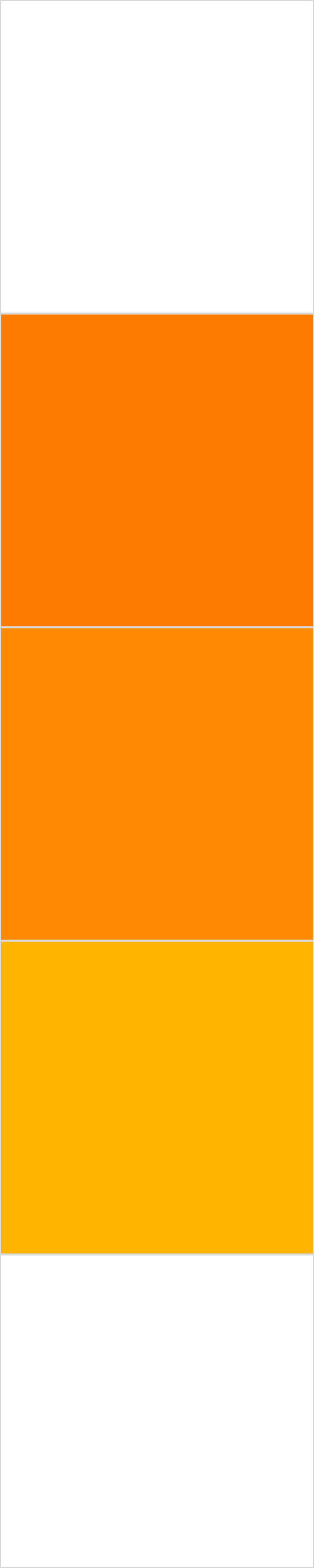}
            }
        },
        2t2/.style={
            frame,
            draw=crimson,
            line width=0.8pt,
            opacity=.9,
            execute at begin node={
                \includegraphics[width=.4cm,height=4cm]{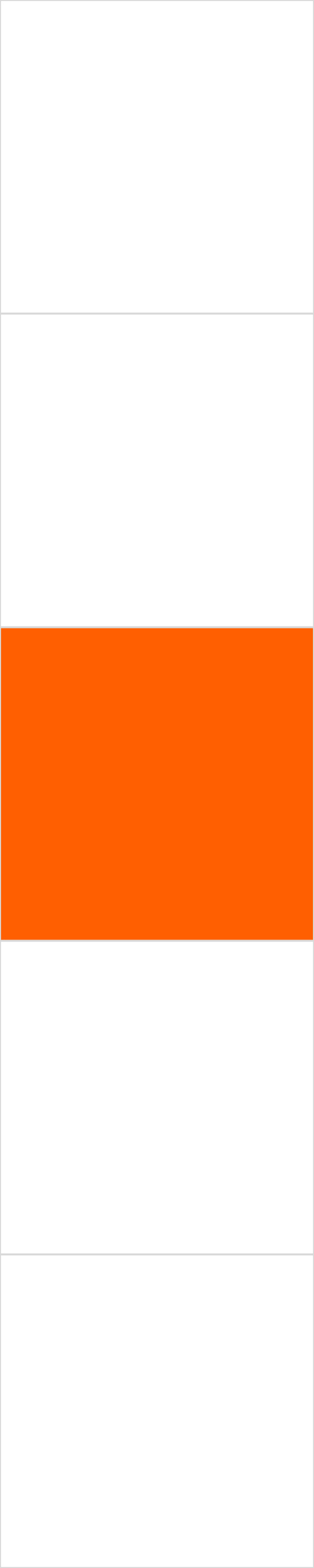}
            }
        },
        2t3/.style={
            frame,
            draw=gray,
            line width=0.8pt,
            opacity=.9,
            execute at begin node={
                \includegraphics[width=.4cm,height=4cm]{rasterSetvoid.png}
            }
        },
        2t4/.style={
            frame,
            draw=TULightYellow,
            line width=0.8pt,
            opacity=.9,
            execute at begin node={
                \includegraphics[width=.4cm,height=4cm]{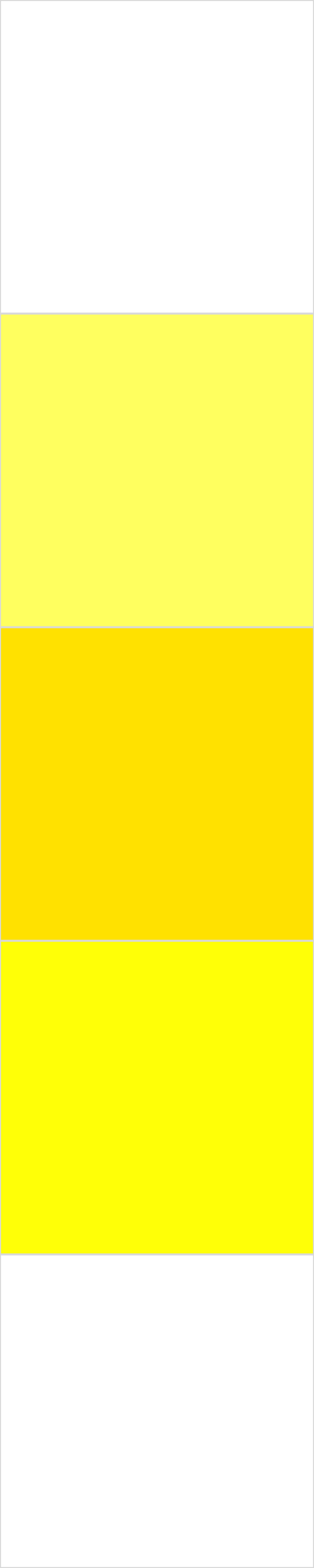}
            }
        },
        2t5/.style={
            frame,
            draw=TULightYellow,
            line width=0.8pt,
            opacity=.9,
            execute at begin node={
                \includegraphics[width=.4cm,height=4cm]{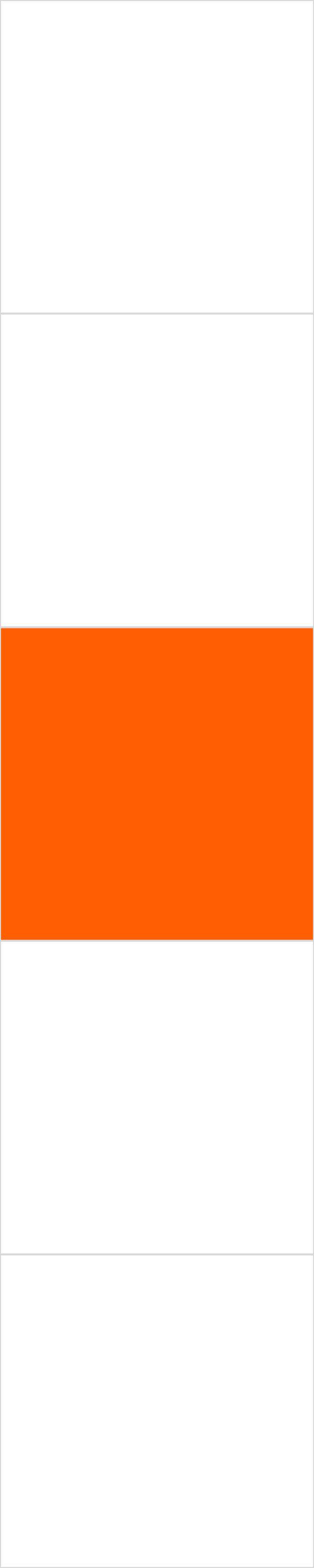}
            }
        },
        2t6/.style={
            frame,
            draw=gray,
            line width=0.8pt,
            opacity=.9,
            execute at begin node={
                \includegraphics[width=.4cm,height=4cm]{rasterSetvoid.png}
            }
        },
        2t7/.style={
            frame,
            draw=TUGreen,
            line width=0.8pt,
            opacity=.9,
            execute at begin node={
                \includegraphics[width=.4cm,height=4cm]{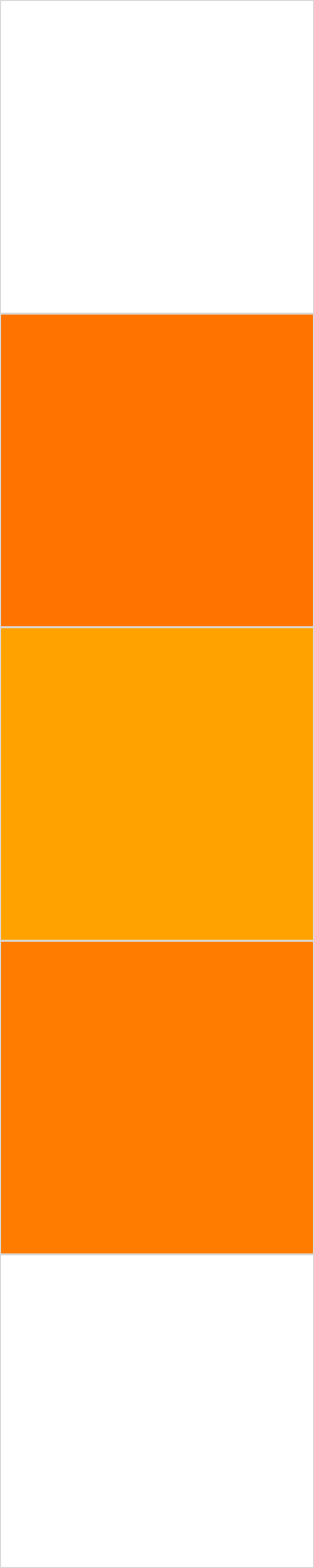}
            }
        },
        2t8/.style={
            frame,
            draw=TUGreen,
            line width=0.8pt,
            opacity=.9,
            execute at begin node={
                \includegraphics[width=.4cm,height=4cm]{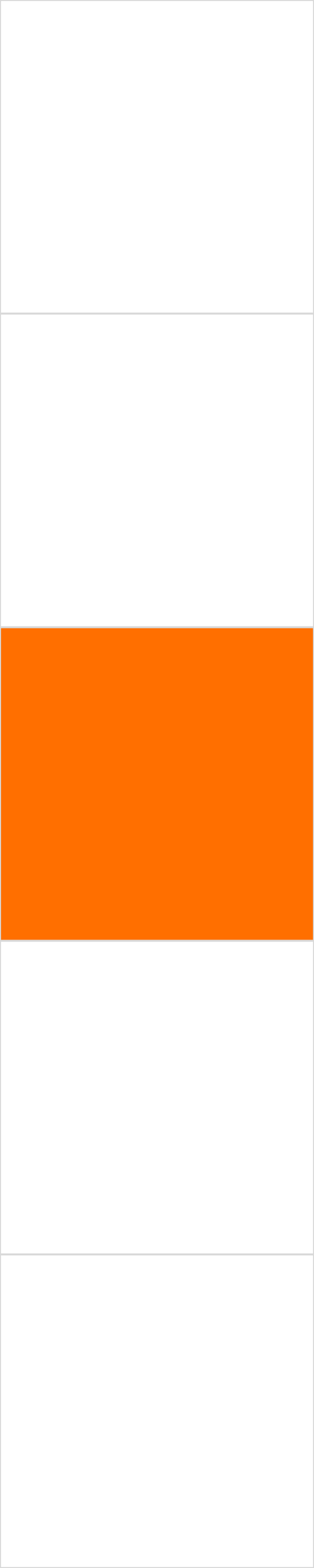}
            }
        },        
    ]
    
    \begin{scope}[xshift=7.5cm, scale=0.8, transform shape]
    \draw[->,thin](0,-0.5)--(0,5)
        node[rotate=90,above,midway]{Spatial position};
    \draw[->,thin](-.5,0)--(8,0)
        node[right,below]{Time};
    \draw[dashed,thin,TUGreen](7.2,2.73)--(6.79,3.18);
    \draw[dashed,thin,TULightYellow](7.2-2.25,2.73)--(6.79-2.25,3.18);
    \draw[dashed,thin,crimson](7.2-4.5,2.73)--(6.79-4.5,3.18);
    \draw[dashed,thin,TUGreen](7.21,1.95)--(6.455,1.15);
    \draw[dashed,thin,TULightYellow](7.21-2.25,1.95)--(6.455-2.25,1.15);
    \draw[dashed,thin,crimson](7.21-4.5,1.95)--(6.455-4.5,1.15);
    
    \foreach \x in {0,...,8}
        \node[2t\x] at (.75*\x+1,2) {};
    \node[crimson]    at (1.8,-0.9) {Training};
    \node[TULightYellow] at (.75*4+1,-0.9) {Validation};

    \node[TUGreen]  at (6.4,-0.9) {Test};
    \draw[-,thin,TUGreen](6.8,2.05)--(6.045,2.85);
    \draw[-,thin,TUGreen](6.79,3.18)--(6.455,3.55);
    \draw[-,thin,TUGreen](6.8,1.28)--(6.045,.45);

    \draw[-,thin,TULightYellow](6.8-2.25,2.05)--(6.045-2.25,2.85);
    \draw[-,thin,TULightYellow](6.79-2.25,3.18)--(6.455-2.25,3.55);    \draw[-,thin,TULightYellow](6.8-2.25,1.28)--(6.045-2.25,.45);

    \draw[-,thin,crimson](6.8-4.5,2.05)--(6.045-4.5,2.85);
    \draw[-,thin,crimson](6.79-4.5,3.18)--(6.455-4.5,3.55);
    \draw[-,thin,crimson](6.8-4.5,1.28)--(6.045-4.5,.45);

    \draw[-,thin](-.5,0)--(8,0);
    \draw[-,thin](-.5,0)--(8,0);
    \end{scope}
    \end{tikzpicture}
\end{figure}
    
We call $\bb{S}=(\bb{X}_i,\bb{Y}_i)_{i \in \Z}$ the stochastic process generating the data described above. $\bb{S}$ is defined on the canonical probability space $(\Omega, \mathcal{F}, \mathbb{P})$ and it is stationary because of the property of the data generating process (\ref{stou}) where $(\bb{X}_i,\bb{Y}_i)$ have values in $\R^D \times \R$ and are identically distributed. $\bb{S}$ inherits the dependence structure of the field $\bb{Z}$ as proven in Appendix \ref{Inher}. Moreover, the geometry of the ambit set enables us to provide a causal interpretation of the forecasts generated using such features in a given spatial position $x^*$ as discussed in Section \ref{casualstructure}.

\begin{algorithm}[H]

\caption{Feature extraction routine for a spatio-temporal position $(t,x^*)$.}

\medskip
\begin{tabular}{lll}
\textbf{Input:} & $(Z_t(x))_{(t,x)\in \mathbb{T}\times\mathbb{L}}$ & \parbox[t]{9.8cm}{Raster data consisting of realizations from a (zero mean) STOU process} \\
& $p$ & \parbox[t]{9.8cm}{Hyperparameter (its default value is equal to $1$)}\\
\end{tabular}

\medskip
\begin{enumerate}
    \item Estimate the parameters $A,c,\lambda$ using equations \eqref{almut1} and \eqref{plug-in}.
    \item Select $a$ as the smallest integer satisfying (\ref{choice})
    \item Define the index sets $I(t_0+ia,x^*),\, i=1,\ldots,\lfloor\frac{N}{a}\rfloor$ following \eqref{index}.
    \item Extract features from $l(t_0+ia,x^*)$ having indices in $I(t_0+ia,x^*)$.
\end{enumerate}

\medskip
\textbf{Output:} $(X_i,Y_i)$ such that 
\begin{align*}
        X_i = (Z_{i_1}(x_1),\ldots,Z_{i_{D}}(x_{D}))^\top \qquad Y_i = Z_{t_0+ia}(x^*),
    \end{align*}
    with $i=1,\ldots,\lfloor\frac{N}{a}\rfloor$.
\label{algorithm1}
\end{algorithm}

\section{MMAF-guided learning for stochastic feed-forward neural networks}
\label{new}
In this section, we describe the training and inference steps in the \emph{MMAF-guided learning workflow} applied to stochastic feed-forward neural networks. We assume throughout that a 2D raster data set is generated by the STOU process (\ref{stou}) with finite second moments.

\subsection{Optimization and Training}
\label{bound}

We work with the following class of stochastic feed-forward neural networks. Let $0$ and $\ell +1$ be the indexes denoting the input and output layers, $l$ the total number of hidden layers (indexed from $1$ to $\ell$), $D$ the dimension of the input space, and
\[
n_0=D,n_1,\ldots,n_\ell,n_{\ell+1}=1
\]
the widths of each layer in the network. We then define for an input  $X_i \in \R^D $
\begin{equation}
\label{feed}
h_{\theta}(X_i)= W^{(\ell+1)}\sigma^{(\ell)}(W^{(\ell)}(\sigma^{(\ell-1)}(W^{(\ell-1)} \ldots \sigma^{(1)}(W^{(1)}X_i+b^{(1)})+b^{(2)} \ldots)+b^{(\ell-1)})+b^{(\ell)}), 
\end{equation}
where the functions $\sigma^{(1)},\ldots, \sigma^{(l)}$ are ReLu activation functions and $$\theta:=(W^{(1)},\ldots, W^{(\ell)}, W^{(\ell+1)}, b^{(1)},\ldots,b^{(\ell)})$$ corresponds to all the weights and bias parameters of the network. We assume in general that $\theta \in \R^d$. The class of stochastic feed-forward neural network we work with is then defined as $\mathcal{H}:=\{ h_{\theta}(\cdot): \theta \sim \mathcal{N}(\mu, diag(\kappa)), \text{for $\mu \in \R^d$ and $\kappa \in \R^d_+$}\}$.

Let $S_m:=(X_i,Y_i)_{i=1}^m$ be the training data set defined in Section \ref{emb},  $\epsilon >0$ be a hyperparameter called the \emph{accuracy level}, we define the truncated absolute loss $L^{\epsilon}:\R\times\R\longrightarrow[0,+\infty]$, as $L^{\epsilon}(x,y):=|x-y|\wedge \epsilon$. For a given realization of the parameter $\theta$, we define the \emph{population risk} $R(h_{\theta}):=\E[L^{\epsilon}(h_{\theta}(\bb{X}),\bb{Y})]$ and the \emph{empirical risk} $r(h_{\theta}):=\frac{1}{m}\sum_{i=1}^m L^{\epsilon}(h_{\theta}(X_i), Y_i)$  for each $m \geq 1$. 

We select now a \emph{reference distribution} $\pi$ on the space $(\R^d, \mathcal{B}(\R^d))$, where $\mathcal{B}(\R^d)$ indicates the Borel $\sigma$-algebra on $\R^d$. The reference distribution provides a structure on the parameter space $\R^d$, which we interpret as our belief that certain parameters are more likely than others to yield better performance. We then learn a distribution over the parameter space \emph{using the observed training data} by solving an optimization problem. We call such distribution a \emph{generalized posterior distribution} $\rho$ on $(\R^d, \mathcal{B}(\R^d))$, and this optimization-centric way of selecting $\rho$ a \emph{generalized Bayesian algorithm} \citep{OptCentric}. Let us first give a more careful definition of the former.

\begin{definition}
    Let $\mathcal{G}_m:=\sigma(\bb{S}_m)$, where $\bb{S}_m$ is a random vector generating the training data $S_m$. We define a generalized posterior distribution $\rho$ as the regular conditional probability distribution of the random variable $\hat{\bb{h}}$ defined on $(\Omega,\mathcal{F}, \mathbb{P})$ with values in $(\R^d, \mathcal{B}(\R^d))$ given $\mathcal{G}_m$, i.e. for all $\omega \in \Omega$ and $E \in \mathcal{B}(\R^d)$
    \[
     \rho(E,\omega):=\mathbb{P}(\hat{\bb{h}} \in E|\mathcal{G}_m)(\omega),
    \]
    where
\begin{itemize}
		\item[a)] for any $E \in \mathcal{B}(\R^d)$, the function $\omega \to \rho(E,\omega) $ is $\mathcal{G}_m$-measurable and a variant of the conditional probability $\mathbb{P}(\bb{\hat{h}} \in E | \mathcal{G}_m)$, i.e.,
		\[
		\rho(E, \cdot) = \mathbb{P}(\hat{\bb{h}} \in E | \mathcal{G}_m)(\cdot) \,\,\, \text{a.s.}
		\]
		\item[b)] for any $\omega \in \Omega$, $\rho(\cdot,\omega)$  is a probability measure on $(\R^d, \mathcal{B}(\R^d))$.
	\end{itemize}
\end{definition}

Second, the optimization problem is defined using the following theoretical result. In \cite{mmafGL}, a so-called PAC Bayesian bound is proven, which is a bound on the \emph{expected generalization gap} between the \emph{average population and average empirical risk} defined for $\omega \in \Omega$ as $$\rho[R(h_{\theta})](\omega)-\rho[r(h_{\theta})]:= \int_{\R^d} R(h_{\theta}) \, \rho(d\theta, \omega) -  \int_{\R^d} r(h_{\theta}) \, \rho(d\theta, \omega) $$ holding with high probability $\mathbb{P}$, i.e., with respect to the distribution of the stochastic process $\bb{S}$ generating the features. We indicate throughout with $\pi[\cdot]$ the expectation with respect to the reference distribution $\pi$. We typically drop the dependence of $\rho$ on $\omega$ and simply write $\rho$ if we are considering it as a probability measure on $(\R^d, \mathcal{B}(\R^d))$.

For all $m\geq 1$, and $\omega \in \Omega$ such that $\rho$ is absolutely continuous  with respect to the reference distribution $\pi$,  the following inequality holds with probability $\mathbb{P}$ at least $1-2\delta$ for $\delta\in (0,1)$:
\begin{small}
\begin{align}
\label{PAC}
	\rho[R(h_{\theta})]  \leq &  \rho[r(h_{\theta})] + \frac{KL(\rho,\pi)+\log \Big(\frac{1}{\delta}\Big)}{\sqrt{m}} +\frac{\epsilon^2}{2\sqrt{m}} \nonumber\\
    &+ \Big( \frac{\epsilon}{\delta} \pi[2(Lip(h_{\theta})D+1)\theta_{lex}(a-p)] (\mathcal{X}^2(\rho,\pi)+1) \Big)^{\frac{1}{2}}.
\end{align}
\end{small}
In the inequality, $\theta_{lex}(a-p)$ is a $\theta$-lex coefficient of the data generating process $\bb{Z}$ (an STOU process in our assumptions) and $Lip(h_{\theta})$ refers to the Lipschitz coefficient of the feed-forward neural network $h_{\theta}$ for a given realization $\theta$ of the parameters. The symbols $KL$ and $\mathcal{X}^2$ stand for the \emph{Kullback-Leibler} and \emph{Chi-square} divergence between two probability measures, respectively. A proof of the inequality (\ref{PAC}) can be found in \cite{mmafGL}. We report in Table \ref{tab:summaryPac} a summary of all parameters, hyperparameters, and symbols appearing in the inequality (\ref{PAC}) for clarity.

\begin{table}[]
     \centering
    \begin{tabular}{c|c}
    \hline
    \hline
    Symbol & Meaning \\
    \hline
    $\rho$ & Generalized posterior distribution\\
    $\pi$ & Reference distribution\\
    $h_{\theta}$ & Feed Forward Neural Network\\
    $Lip(\cdot)$ & Lipschitz constant\\
    $R(\cdot)$ & Population risk\\
    $r(\cdot)$ & Empirical risk\\
    $KL$ & Kullback-Leibler divergence\\
    $\mathcal{X}^2$ & Chi-square divergence\\
    $\epsilon$ & Accuracy level\\
    $m$ & Length of the training set\\
    $\delta$ & Confidence level\\
    $D$ & Dimension of the input\\
    $a$ & Temporal distance between $\bb{Y}_i$ and $\bb{Y}_{i+1}$ (parameter)\\
    $p$ & Past time steps used to design $\bb{X}_i$ (hyperparameter)\\
    $\theta_{lex(a-p)}$ & $\theta$-lex coefficient of $\bb{Z}$\\
    \end{tabular}
    \caption{summary of all parameters, hyperparameters, and symbols appearing in the inequality (\ref{PAC})}
    \label{tab:summaryPac}
\end{table}

We could then select a generalized posterior distribution $\rho$ that minimizes the right-hand side of the inequality (\ref{PAC}). The generalized posterior selected in this way achieves the best possible generalization performance on data generated by $\bb{S}$, since the PAC-Bayesian bound certifies that the average generalization gap is minimal with high probability. If, for example, $\delta=0.025$, the algorithm is expected to generalize well, i.e., it has a low expected generalization gap with probability $95\%$. Note that, given an accuracy level $\epsilon$, the bound of the expected generalization gap in (\ref{PAC}) doesn't have to overcome this threshold to give a generalization guarantee. We call such bounds \emph{non-vacuous}.
The ideal optimization routine will then be to select, for a given observed data set (which in our notations corresponds to one $\omega \in \Omega$), a generalized posterior distribution $\rho$ belonging to the set of all probability measures defined on $(\R^d, \mathcal{B}(\R^d))$ which minimizes the following target function:
\begin{equation}
\label{no}
\rho[r(h_{\theta})] + \frac{KL(\rho,\pi)+\log \Big(\frac{1}{\delta}\Big)}{\sqrt{m}} + \Big( \frac{\epsilon}{\delta} \pi[2(Lip(h_{\theta})D+1)\theta_{lex}(a-p)] (\mathcal{X}^2(\rho,\pi)+1) \Big)^{\frac{1}{2}}.
\end{equation}
Note that the addend $\frac{\epsilon^2}{2\sqrt{m}}$ appearing in the inequality (\ref{PAC}) is not included because it is a constant factor and does not depend on $\theta$.  A theoretical solution for (\ref{no}) is an open problem. We then turn to a numerical optimization. The problem with optimizing (\ref{no}) with respect to $\rho$ in the space of all probability measures on $(\R^d, \mathcal{B}(\R^d))$ lies in the simultaneous presence of the \emph{KL} and \emph{Chi-square} divergence in the bound.
Especially, optimizing the \emph{Chi-square} leads to several numerical instabilities \citep{GD21}. Therefore, we typically approximate it rather than base an optimization routine on its original definition; see \cite{optchi2, optchi} for some exemplary solutions to such issue. In this work, we \emph{linearize} the \emph{Chi-square} divergence using a Taylor approximation \cite{ChiApprox}, that is we approximate $\chi^2(\rho,\pi) \simeq 2KL(\rho,\pi)$. Next, we constrain the $\theta$-lex coefficient appearing in the bound by choosing the parameter $a$ in the design of the embedding such that for $\delta=0.025$
\begin{equation}
\label{choice}
\theta_{lex}(a-p) \sim \exp(-\lambda (a-p)) \leq \frac{\delta}{ 2 \,m \, \epsilon}.   
\end{equation}
We know in fact from (\ref{boundT}) that the $\theta$-lex coefficients have an exponential decay if we assume that the data are generated by an STOU process.
In conclusion, training in our methodology means selecting a generalized posterior distribution  $ \rho \in \{\mathcal{N}(\mu,diag(\kappa))$ for $\mu \in \R^d$ and $\kappa \in \R^d_+\}$ which minimizes the target function
\begin{align}
\label{target}
\rho[r(h_\theta)] +  \frac{KL(\rho,\pi) + [(2KL(\rho,\pi)+1) (\pi[Lip(h_\theta)]D+1)]^{1/2}}{\sqrt{m}},
\end{align}
with $\pi\sim N(0,sI_d)$, $s\in\R_+$, being a multivariate Gaussian distribution with diagonal covariance matrix. The generalized posterior distribution obtained by minimizing the target function (\ref{target}) loses its interpretation as the right-hand side of a PAC-Bayesian bound. However, our experiments in Section \ref{sec3} show that we typically obtain an approximation $\rho^*$ of the minimum to which corresponds, for our choice of the parameter $a$, a non-vacuous PAC Bayesian bound holding with probability of at least $95\%$. We give in Algorithm \ref{algorithm2} a description of the training routine. Note that, in general, the algorithm can use batch learning (where the batches are taken in chronological order) and epochs.



\begin{algorithm}
\caption{MMAF-guided learning training algorithm for stochastic feed-forward neural networks for a spatial position $x^*$.}

\medskip
\begin{tabular}{lll}

\textbf{Input:} & $S_m$=$\{(X_i, Y_i)\}_{i=1}^m$  & \parbox[t]{9.6cm}{Training data set.}\\
& $m^{\prime}$ & Batch size.\\
& $N^{\prime}$ & Number of batches.\\
& $\epsilon,\eta\in \mathbb{R}_+$ & \parbox[t]{9.6cm}{Accuracy level, Adam learning rate.}\\
& $\pi \sim \mathcal{N}(0,sI_d)$ & \parbox[t]{9.6cm}{Reference distribution, $s\in\R_+$.}\\
& $h_{\theta}$ & \parbox[t]{7cm}{Feed-forward neural network architecture with parameters $\theta \in \R^d$.}\\
& $\rho^{0}\sim \mathcal{N}(\mu^0,diag(\kappa^0))$ & \parbox[t]{6.6cm}{Initialization of the generalized posterior distribution with $\mu^0:=0\in\R^d,\kappa^0\in\R^d_+$.}\\
\end{tabular}

\;\;\;\textbf{for $\iota= 0, \ldots, N^{\prime}-1$ do:}
\medskip
\begin{enumerate}
    \item Estimate $Lip(h_{\theta})$ using the upper bound
    \begin{equation}
    \label{algorithm:LipBound}
        Lip(h_{\theta})\leq\prod_{i=1}^{\ell+1} ||W^{(i)}||_2
    \end{equation}
    \item Compute the Monte Carlo estimator of $\pi[Lip(h_{\theta})]$ using the bound (\ref{algorithm:LipBound}):
    \[
    \pi[Lip(h_{\theta})]\simeq \frac{1}{1000}\sum^{1000}_{j=1} \prod_{i=1}^{\ell+1} ||W_j^{(i)}||_2,
    \]
    where $\theta^j=(W^{(1)}_j,\ldots, W^{(\ell)}_j, W^{(\ell+1)}_j, b^{(1)}_j,\ldots,b^{(\ell)}_j)$ is drawn from $\pi$ for $j=1,...,1000$.
    \item Compute $\nabla\rho^{\iota}[r(h_\theta)]$ employing pathwise gradient estimator:
    \begin{equation}\label{grad1}
     \nabla_{(\mu^{\iota},\kappa^{\iota})}\rho^{\iota}[r(h_\theta)] \simeq \nabla r(h_{{\theta^{\iota}}_\cdot}),\quad {\theta^{\iota}} \text{ single draw from }\rho^{\iota}, 
    \end{equation}
    where
    \(
    r(h_{\theta^{\iota}})=\frac{1}{m^{\prime}}\sum_{j\in \mathcal{J}_{\iota}}L^{\epsilon}(h_{\theta^{\iota}}(X_j),Y_j),
    \)
    and $\mathcal{J}_{\iota}=\{\iota m^\prime\ldots,(\iota+1)m^\prime\}$ is the $\iota$-th batch.
    \item Compute 
    \(
    KL(\rho^{\iota},\pi)=\frac{1}{2}\sum_{i=1}^d \left(\log\left(\frac{s}{\kappa^{\iota}_i}\right) -1 +\frac{\kappa^{\iota}_i}{s} + \frac{(\mu^{\iota}_i)^2}{s}\right)
    \)
    \item Compute 
    \begin{small}
      \begin{equation}\label{grad2}
     \nabla_{(\mu^{\iota},\kappa^{\iota})} \left(\frac{KL(\rho^{{\iota}},\pi) + ((2KL(\rho^{{\iota}},\pi)+1) (\frac{D}{1000}\displaystyle{\sum^{1000}_{j=1} \prod_{i=1}^{\ell+1}} ||W_j^{(i)}||_2 +1))^{1/2}}{\sqrt{m}} \right)   
     \end{equation}  
    \end{small}
    \item Update $\rho^{\iota+1}$ using Adam update rule with $f(\theta^{\iota}):=$(\ref{grad1})+(\ref{grad2}) and learning rate $\eta$.
\end{enumerate}
\;\;\;\textbf{end for}

\medskip
\begin{tabular}{lll}
\textbf{Output:} & $\rho^{*}$ &  
\end{tabular}
\label{algorithm2}
\end{algorithm}

\subsection{Causal Forecast and Inference}
\label{casualstructure}
Let $X_i$ be an input belonging to the validation or test data set related to the spatial position $x^*$, i.e, $X_i = (Z_{i_1}(x_1),\ldots,Z_{i_{D}}(x_{D}))^\top $ for $(i_s,x_s) \in I(t_0+ia, x^*)$ and $i=m+1,\ldots, \lfloor \frac{N}{a} \rfloor$.
Then, following the definition given in (\ref{lightcone2}), the input vector $X_i$ can be used to forecast the values of the field $\bb{Z}$ which corresponds to the spatial-time positions  
$$
A(X_i)^+:=\{(t,s) \in \bigcap_{s=1}^D A_{i_s}(x_s)^+\}
$$
Note that the spatial-time position $(t_0+ia,x^*)$ belongs to $A(X_i)^+$ and can then be forecasted starting from the input $X_i$, see Figure \ref{future}.
\begin{figure}[!htbp]
\centering
\begin{tikzpicture}
\def\dx{1.3}
    \def\dy{0.7}
    \def\r{0.9pt}
    
    \filldraw[fill=yellow!40, draw=yellow, opacity=0.5]
        (\dx-4*\dx,\dy)--(3*\dx-4*\dx,3*\dy)--(3*\dx-4*\dx,-1*\dy)--cycle;
    
    \filldraw[fill=orange!40, draw=orange, opacity=0.3]
        (\dx-4*\dx,2*\dy) 
        --(3*\dx-4*\dx,4*\dy) 
        --(3*\dx-4*\dx,0*\dy)
        --cycle;
    
    \filldraw[fill=pink!40, draw=pink, opacity=0.5]
        (\dx-4*\dx,3*\dy)--(3*\dx-4*\dx,5*\dy)--(3*\dx-4*\dx,1*\dy)--cycle;
    
    \filldraw[fill=TUGreen!40, draw=TUGreen, opacity=0.5]
        (\dx-4*\dx,1*\dy)
        --(\dx-4*\dx,3*\dy)
        --(2*\dx-4*\dx,2*\dy)
        --cycle;
    \filldraw[fill=crimson!40, draw=red, opacity=0.5]
        (2*\dx-4*\dx,2*\dy)--(3*\dx-4*\dx,3*\dy)--(3*\dx-4*\dx,1*\dy)--cycle;

    \foreach \i in {1}{
        \foreach \j in {1,...,3}{
            \fill (\i*\dx-4*\dx, \j*\dy) circle (\r);
        }
    }

    \foreach \i in {1}{
        \foreach \j in {-1,0,4,5}{
            \fill[fill=gray!75] (\i*\dx-4*\dx, \j*\dy) circle (\r);
        }
    }
    \foreach \i in {2}{
        \foreach \j in {2}{
            \fill[fill=crimson] (\i*\dx-4*\dx, \j*\dy) circle (\r);
        }
    }

    \node[above, font=\tiny] at (\dx-4*\dx,1*\dy) {$(i_1,x_1)$};
    \node[above, font=\tiny] at (\dx-4*\dx,2*\dy) {($i_2$,$x_2$)};
    \node[above, font=\tiny] at (\dx-4*\dx,3*\dy) {$(i_3,x_3)$};
    \node[right, font=\tiny] at (2*\dx-4*\dx,2*\dy) {$(t_0+ia,x^*)$};

    \draw[->] (-4.8,-1.2) -- (-1,-1.2) node[below,midway] {Time};
    
    \draw[->] (-4.8,-1.2) -- (-4.8,3.8) node[rotate=90,above,midway] {Spatial position};

\end{tikzpicture}
\caption{The black dots indicate the spatial-temporal index grid $\R \times \R$, corresponding to an input feature $X_i=(Z_{i_1}(x_1),Z_{i_2}(x_2),Z_{i_3}(x_3)$ for $(i_s,x_s) \in I(t_0+ia,x^*)$. The pink area identifies the future cone of influence related to the spatial-time position $(i_1,x_1)$, the orange area the future cone of influence from $(i_2,x_2)$, and the yellow one the future cone of influence from $(i_3, x_3)$. All three spatio-temporal positions together influence the future realizations of the field at the spatio-temporal points in the red area, which is a subset of $\R \times \R$.}
\label{future}
\end{figure}

In MMAF-guided learning, we forecast the values of the field $\bb{Z}$ at the time points $t_0+ia$ for $i=m+1,...,\lfloor \frac{N}{a} \rfloor$ at the spatial position $x^*$. Other forecasting schemes can be investigated based on the input features' ability to predict future values of the data-generating process. For example, \cite{mmafGL} investigates a scheme where the forecasts correspond to one-step ahead with respect to the information contained in the training data set. The latter scheme also has a causal interpretation, and it can be extended to multiple time steps into the future. We leave the analysis of such forecasting schemes to future applications of MMAF-guided learning.

In Section \ref{sec3}, we employ a training data set $$\mathcal{T}=(S_m^1, S_m^2, \ldots, S_m^{R})$$ corresponding to $R$ different spatial positions in a raster. At the end of the training step (parallelized with respect to each spatial position), we obtain the set of generalized posterior distributions $\{\rho^{*,1},\rho^{*,2},\ldots, \rho^{*,R} \}$. The latter are used in inference mode to compute ensemble forecasts for each spatial position at times $t=t_0+i a$ for $i=m+1, \ldots, \lfloor \frac{N}{a} \rfloor$. In our experiments, we call $H \in \{0,\ldots, \lfloor \frac{N}{a} \rfloor - i \}$ the \emph{terminal time horizon} for which forecasts are possible starting from an observation at time $t=t_0+ia$ and an observed raster dataset of length $N$. We describe below the Algorithm \ref{algorithm3} which generates ensemble forecasts (in a given spatial position $x^r$ for $r\in\{1,\ldots,R\}$) with $J$ members across $H$-time horizons. 

\begin{algorithm}[h!]
\caption{Inference step: ensemble forecasts generation across multiple time horizons for a single spatial position $x^r$.}

\medskip
\begin{tabular}{lll}
\textbf{Input:} & $\rho^{*,r}$ & \parbox[t]{11cm}{Generalized posterior distribution on $(\R^d, \mathcal{B}(\R^d))$}\\
& $X_i^r$  & \parbox[t]{11cm}{Input belonging to the validation or test set related to the spatial position $x^r$}\\
& $H$  & \parbox[t]{11cm}{Terminal time horizon}\\
& $h_{\theta^{r,j}}$ & \parbox[t]{11cm}{Feed-forward neural network architecture with parameters $\theta^{r,j} \in \R^d$}\\
& $J$  & \parbox[t]{11cm}{Number of ensemble members}\\
&&\\
\end{tabular}

\;\;\;\textbf{for $j= 1 \ldots J$ do:}
\medskip
\begin{enumerate}
    \item Draw a sample $\theta^{r,j}$ from $\rho^{*,r}$.\\
    \item Compute $\left(h_{\theta^{r,j}}(X_{i + h}^r),\ldots,h_{\theta^{r,j}}(X_{i+h}^r) \right)_{h=0}^H$, the $j$-th member of the ensembles.\\
\end{enumerate}
\;\;\;\textbf{end for}

\medskip
\begin{tabular}{lll}
\textbf{Output:} &  $\left[\left(h_{\theta^{r,1}}(X_{i+h}^r),\ldots,h_{\theta^{r,J}}(X_{i+h}^r) \right)\right]_{h=0}^H$ &
\end{tabular}
\label{algorithm3}
\end{algorithm}

We claim that the ensemble members for a fixed spatial position $x^r$ and horizon $h\in\left\{ 0, \ldots, \lfloor \frac{N}{a}\rfloor-i \right\}$ for $i=m+1,\ldots,\lfloor \frac{N}{a}\rfloor $, are samples from a distribution that approximates $\mathbb{P}(\bb{Y}_{i+h}^r | \bb{X}_{i+h}^r=X_{i+h}^r)$, i.e., the \emph{conditional predictive distribution of the data} in our setup which is unknown and, in general, not of Gaussian type. We test this claim in the next section by employing the \emph{Continuous Ranked Probability Score} \cite{gneiting07}, a probabilistic metric that assesses the consistency between the ensemble members' distributions and the observed value.

\section{Experiments}
\label{sec3}

\subsection{Data set and Training set-ups}
\label{preprocessing}
We employ three datasets in our experiments on stochastic feed-forward neural networks. Two are synthetic and sampled from an STOU process with a Gaussian and a normal inverse Gaussian distributed L\'evy basis, generated using the diamond grid algorithm \citep{STOU}. We refer to the synthetic datasets as GAU and NIG throughout, both of which consist of raster datasets with $10$ spatial positions, $8$ of which are used in the inference step, and $N=10^6$ temporal records. The last one is a real dataset of \emph{hourly 2m surface temperature} records (2mT) from ERA5, the fifth generation of atmospheric data reanalysis provided by the ECMWF \citep{Era5}. The 2mT dataset contains $101$ spatial components over the equator, corresponding to a longitude west in degrees of $ 57.5, \ldots,82.5$,  whereas the temporal component consists of hourly sampled values of the temperature from January 1st, 2016 until December 31st, 2025, for a total of $N=96.431$ temporal records. The variable considered is the air temperature 2 meters above the surface of land, sea or in-land waters. It is calculated by interpolating between the lowest model level and the Earth's surface, taking into account the atmospheric conditions, and it is measured in Kelvin. $8$ spatial locations between longitudes west in degrees $80.25,\ldots, 82.25$  are used in the inference step. For all data sets, we set $t_0=0$ throughout.

Regarding the 2mT dataset, we first preprocess the data to remove any deterministic component. It is assumed that the function $\mu_t(x)$ in \eqref{decomposition} has the following form
\begin{equation*}
    \mu_t(x) = a_t(x)t + a_s(t)x + b_t(x) + b_s(t) + s_d(t,x) +s_a(t,x),
\end{equation*}
where the term $a_t(x)t + a_s(t)x + b_t(x) + b_s(t)$ consists of a sum of linear trend components along each of the raster coordinates, in line with \citep{HA21}, while $s_d(t,x),s_a(t,x)$ correspond to the daily and annual seasonal components, respectively, modeled via harmonic regression \citep{HR73}: 
\begin{equation*}
\begin{split}
    & s_d(t,x):= \sum_{i=1}^3 s^1_{d,i}(x) \sin\left(\frac{2\pi i}{24}t\right) + s^2_{d,i}(x) \cos\left(\frac{2\pi i}{24}t\right), \\
    & s_a(t,x):= \sum_{i=1}^5 s^1_{a,i}(x) \sin\left(\frac{2\pi i}{8760}t\right) + s^2_{a,i}(x) \cos\left(\frac{2\pi i}{8760}t\right).
\end{split}
\end{equation*}

To initiate the workflow, we need to determine the embedding for each spatial position of interest. This is done by estimating the values of the \emph{mean reverting parameter} $A$, the \emph{speed of information propagation} $c$, and the decay rate $\lambda$ of the $\theta$-lex coefficients of $\bb{Z}$, following the estimation procedure in Appendix \ref{est}. The estimations of such values are listed in Table \ref{pre-est} for all three datasets. Such estimations allow us to learn the dependence and causal structure of the data under our modeling assumption. We then select the hyperparameters $p$ and $\epsilon$ following the theoretical framework of \cite{mmafGL}, which provides the necessary insights to tune them accordingly: that is, $p$ is chosen equal to $1$ for ease of computation and $\epsilon=3$. We then select the parameter $a$ according to rule (\ref{choice}), yielding $a=8$ for the GAU and NIG data sets and $a=146$ for the 2mT data set. All obtained features are then partitioned along the temporal dimension following the characterization given in Section \ref{emb} for the training, validation, and test data sets. The length of the training data set is $m=124899$ for the GAU/NIG, whereas $m=559$ for the 2mT data set. Moreover, the validation set consists of a single example, and the test sets contain $100$ and $40$ examples for GAU/NIG and 2mT data, respectively.

We refer to each feed-forward architecture using the symbol $w^\ell$, where we assume that the inner layers have the same width, i.e., $w:=n_1=\ldots=n_{\ell}$ following the notations of Section \ref{bound} and $\ell$ is the number of the hidden layers. In our study, we consider six networks, denoted by $10^2$, $10^5$, $30^2$, $100^2$, $300^2$, and $800^3$. The training routine is performed in parallel across $8$ spatial coordinates. We perform batch learning for the GAU and NIG datasets with $m^{\prime}=1000$ and run the algorithm for $30$ and $40$ epochs, respectively. In the case of the 2mT data set, we use the optimization routine without batch learning and run the algorithm for $5000$ epochs. The stopping rule of our training is based on the convergence study in Appendix \ref{convergence}.
\begin{table}[h]
\begin{center}
\begin{tabular}{l|l|l|l|l|l}
\textbf{\textbf{Data Set}}  &\textbf{$A^*$}&\textbf{$c^*$}&\textbf{$\lambda^*$}& m & a\\
\hline 
GAU & 3.851 & 1.013 & 1.896 & 124899 & 8\\
\hline
NIG &3.869 & 0.998& 1.938 & 124899 & 8\\
\hline
2mT & 0.163 & 1.014 &
0.081 & 559 & 146\\
\end{tabular}
\end{center}
\caption{Estimation of parameters $A$, $c$, $\lambda$, size of the training data set $m$, and chosen parameter $a$ following rule (\ref{choice}).}
\label{pre-est}
\end{table}

In our experiments, we initialize the mean of the reference and posterior distributions to the zero vector, and initialize the covariance matrix of the posterior distribution to $\frac{1}{4} I_d$, where $I_d$ indicates the identity matrix of dimension $d \times d$. The covariance matrix of the reference distribution is \emph{validated} following the procedure described in Section \ref{choice_prior}.


\subsection{Evaluation Metrics}
\label{metrics}
For each spatial position $x^r$ for $r=1,\ldots, R$  and time $ia$ for $i=m+1,\ldots, \lfloor \frac{N}{a} \rfloor $, we can determine ensemble forecasts for multiple-time horizons up to $H \in \{0,\ldots, \lfloor \frac{N}{a} \rfloor - i \}$ terminal time horizons. We then indicate the related ensemble forecasts with $J$ members across $H$-time horizons for the spatial position $x^r$ with $\left[\left(h_{\theta^{r,1}}(X_{i+h}^r),\ldots,h_{\theta^{r,J}}(X_{i+h}^r) \right)\right]_{h=0}^H$, and with $Y_{i+h}^r$ the observed values of the field $\bb{Z}$ in the spatial time point $((i+h)a,x^r)$. In our analysis, we consider two evaluation metrics.

The first is the \emph{Mean Squared Error} (MSE) across the raster, a deterministic metric, defined at each time horizon $h$ as
\begin{align*}
MSE_h(h_{\theta}) = \frac{1}{R} \sum_{r=0}^{R} \frac{1}{J} \sum_{j=1}^{J} (h_{\theta^{r,j}}(X_{i+h}^r)-Y_{i+h}^r)^2
\end{align*}
Note that all values of the MSE reported in our tables below are averaged across all the considered time horizons.

The second metric we employ is the \emph{Continuous Ranked Probability Score} (CRPS), a probabilistic metric used to assess the calibration and sharpness of ensemble forecasts \citep{gneiting05,gneiting07}, which is specified in terms of the cumulative distribution function (cdf) related to the predictive distribution $\mathbb{P}(\bb{Y}_{i+h}^r|\bb{X}_{i+h}^r=X_{i+h}^r)$ for a given time horizon $h$ and spatial position $x^r$. It allows us to evaluate if the ensemble forecast $\left(h_{\theta^{r,1}}(X_{i+h}^r),\ldots,h_{\theta^{r,J}}(X_{i+h}^r) \right)$ is consistent with respect to the observed value $Y_{i+h}^r$, and it is defined as 
\begin{align}
\label{theocrps}
CRPS\Big(\mathbb{P}(\bb{Y}_{i+h}^r|\bb{X}_{i+h}^r=X_{i+h}^r),Y_{i+h}^r\Big) = \mathbb{E}[\bb{Y}-Y_{i+h}^r] - \frac{1}{2}\mathbb{E}[\bb{Y} - \bb{Y}^{\prime} ],\, 
\end{align}
where the random variables $\bb{Y}$ and $\bb{Y}^{\prime}$ are independent random variables distributed with respect to the (conditional) predictive distribution $\mathbb{P}(\bb{Y}_{i+h}^r|\bb{X}_{i+h}^r=X_{i+h}^r)$. In our framework, it is not possible to have access to the latter, so we compute the empirical counterpart of (\ref{theocrps}) using the ensemble forecast $\left(h_{\theta^{r,1}}(X_{i+h}^r),\ldots,h_{\theta^{r,J}}(X_{i+h}^r) \right)$, i.e.

$$
\widehat{CRPS}^r_h(h_{\theta}) = \left( \frac{1}{J}\sum_{j=1}^J |h_{\theta^{r,j}}(X_{i+h}^r) - Y_{i+h}^r| - \frac{1}{J^2}\sum_{j_1=1}^J\sum_{j_2=1}^J |h_{\theta^{r,j_1}}(X_{i+h}^r) - h_{\theta^{r,j_2}}(X_{i+h}^r)| \right)
$$

Note that all the CRPS values reported in our tables below are averaged across the time horizons and the spatial positions considered.

\subsection{Validation of the reference distribution}
\label{choice_prior}

We set up the following grid for the selection of the covariance matrix of the reference distribution, 
\begin{align*}
\left\{s^{-1} I_d: s \in \{10,30,50,70,90,110,130,150,170,190,210\}\right\}.
\end{align*}
We then validate the choice of the covariance matrix with respect to the parameter $s$ by minimizing the average CRPS score computed on the validation set. To the best of our knowledge, this is the first workflow where a CRPS score is used to validate the selection of a reference distribution. Several alternative strategies exist for choosing a reference distribution in PAC-Bayesian learning \citep{Pac-intro, dziugaite}. However, such strategies have been mostly tested on random samples (i.e., independent and identically distributed data).
We demonstrate in our experiments that the validation strategy above enables us to generate \emph{ensemble forecasts} for the test set with CRPS scores of the same order of magnitude as those observed in validation. We refer to our forecasts to \emph{remain calibrated across multiple time horizons} when the latter happens. 

We can also draw a parallel between a well-known post-processing procedure used in weather forecasting, the EMOS \citep{gneiting05, Lerch18}, and our validation methodology. In the former, a minimum CRPS estimation (that makes use of observed data and ensemble forecasts produced by \emph{numerical weather forecast models}) is used to produce calibrated ensemble forecasts in a given spatial-time position, whereas in the latter, we use CRPS minimization to select a prior distribution, which gives us a calibrated ensemble forecast on the validation set.

\subsection{Baseline Models}
\label{baselines}

\begin{table}[h!]
\centering
\small
\caption{Description of the architectures used in the experiments}
\label{tab:model_parameters}

\begin{tabular}{lp{3.5cm}c}
\hline
Model & Architecture & Param. \\
\hline

SFFN & $[10^2]$ & $\sim10^2$\\
SFFN & $[10^5]$ & $\sim5\cdot10^2$\\
SFFN & $[30^2]$ & $\sim10^3$\\
SFFN & $[100^2]$ & $\sim10^4$\\
SFFN & $[300^2]$ & $\sim10^5$\\
SFFN & $[800^3]$ & $\sim10^6$\\

ConvLSTM & 2($k\times1$)-1($k\times k$) & $\sim10^2$ \\
ConvLSTM & 8($k\times1$)-1($k\times k$) & $\sim5\cdot10^2$ \\
ConvLSTM & 15($k\times1$)-1($k\times k$) & $\sim10^3$ \\
ConvLSTM & 150($k\times1$)-1($k\times k$) & $\sim10^4$ \\

ConvGRU & 2($k\times1$)-1($k\times k$) & $\sim10^2$ \\
ConvGRU & 10($k\times1$)-1($k\times k$) & $\sim5\cdot10^2$ \\
ConvGRU & 18($k\times1$)-1($k\times k$) & $\sim10^3$ \\
ConvGRU & 180($k\times1$)-1($k\times k$) & $\sim10^4$ \\

DiffSTG & ($T_h,T_f$)-1 & $\sim10^3$ \\
DiffSTG & ($T_h,T_f$)-4 & $\sim10^4$ \\
DiffSTG & ($T_h,T_f$)-11 & $\sim10^5$ \\
DiffSTG & ($T_h,T_f$)-40 & $\sim10^6$ \\

\hline
\end{tabular}

\vspace{1mm}

\begin{flushleft}
\footnotesize
$\ell(k\times k)$: \emph{$\ell$ is the number of hidden layers and $k$ refers to the kernel size.}\\
$(T_h,T_f)$-$\ell$: \emph{following \cite{HEY25}, $T_h$ and $T_f$ represent the number of historical and future graph feature states, respectively; $\ell$ is the number of hidden layers in the UGnet.}\\
$[w^{\ell}]$: \emph{$\ell$ is the number of hidden layers of the feed-forward architecture, and $w$ is the width of each hidden layer.}
\end{flushleft}

\end{table}

We compare the performance of stochastic feed-forward neural networks (SFFN) trained using MMAF-guided learning in the following section against several baseline models used in probabilistic forecasting. We describe them briefly in this section.
\begin{itemize}
    \item \textbf{ConvLSTM} is an auto-regressive model that combines different long short-term memory (LSTM) modules, and a recurrent neural network (RNN) module. These elements are combined together with convolution layers and a Gaussian head to produce a multivariate (Gaussian distributed) ensemble forecast output (\cite{XZH15,LQQ21} provide a detailed description of the architecture).
    \item \textbf{ConvGRU} is an auto-regressive model that combines different gate recurrent units (GRU) with operations of convolution, followed by a Gaussian head, in order to generate a (Gaussian distributed) ensemble forecast output, see \cite{TLX20} for further details.
    \item \textbf{DiffSTG} is a non-autoregressive diffusion model for spatio-temporal graphs forecasting that makes use of a denoising UGnet network, as described in detail in \cite{HEY25}.   
\end{itemize}

The baselines are defined by considering an increasing number of hidden layers, as reported in Table \ref{tab:model_parameters}. To ensure a fair comparison across methodologies on the same time horizons, the number of past features used to generate forecasts is fixed to $a$ for both the ConvLSTM and ConvGRU models. The kernel size $k=2pc+1$ through all the simulations, where $c=1$ for the GAU and the NIG datasets, $c=$ for the 2mT datasets, and $p=1$. 
The DiffSTG model is configured to use $T_h=p$ historical graph feature states to predict a single graph feature state, i.e., $T_f=1$, where $p$ is the hyperparameter reported in Table \ref{tab:summaryPac}.



\section{Results and Discussions}
\label{sec4}
In this section, we present the results for the various
experiments discussed in the previous sections. The source Python code for the algorithms in our workflow (feature extraction, training, validation, and inference) is available in the \href{https://github.com/Leonardo-Bardi/mmaf_guided_learning.git}{GitHub repository}. The experiments were run on a machine with access to 256GB RAM, 48 GB NVIDIA A40, and 2 AMD EPYC 7313 16-Core Processors.

\subsection{MMAF-guided learning applied to stochastic-feed forward neural networks}
We present in Table \ref{tab:workflow} the results related to the training, validation, and testing for the stochastic neural networks introduced in Section \ref{preprocessing}. 
In the column \textbf{Training}, we find the value of the target function (\ref{target}) and of the right-hand side of the PAC Bayesian bound (\ref{PAC}) (for a confidence level $\delta=0.025$) computed with respect to the generalized posterior distributions obtained at the end of the training routine. These results indicate that Algorithm \ref{algorithm2} mostly selects posterior distributions at each spatial position that correspond to non-vacuous PAC Bayesian bounds, i.e., they are less than $\epsilon=3$ in our experiment and therefore generalize well with probability at least $95\%$. We obtain comparable performance across datasets and architectures, except for the architectures $300^2$ (just in the case of the 2mT daraset) and $800^3$, for which we obtain vacuous PAC Bayesian bounds, indicating that overparameterization with respect to network width takes a toll on the generalization performance of our posterior distribution. Testing the performance of MMAF-guided learning with respect to overparameterized deep neural networks is outside the scope of the present paper.

In the columns related to the \textbf{Validation} of our methodology, we report the value of the selected reference distribution along with the corresponding CRPS and MSE. Finally, in the columns related to the \textbf{Test}, we find the value of the CRPS and MSE averaged with respect to the time horizons. We note that the CRPS values are, in most cases, of the same order of magnitude and within the same range as those observed in the validation set, indicating that the ensemble forecast remains calibrated in time. On the other hand, the MSE increases during evaluation on the test set, indicating that the forecasting quality degrades over time.

\begin{table}
\centering
\small
\resizebox{0.99\columnwidth}{!}{ 
\begin{tabular}{c c|cc|ccc|ccc}
\toprule
\multicolumn{2}{c}{} & \multicolumn{2}{c}{\textbf{Training}} & \multicolumn{3}{c}{\textbf{Validation}} & \multicolumn{3}{c}{\textbf{Test}} \\
Dataset & Architecture & Target (\ref{target}) & PAC bound (\ref{PAC}) & $\pi$ & CRPS & MSE & Horizons & CRPS & MSE  \\
\midrule
\multirow{6}{*}{GAU}
 & $[10^2]$ & 0.1131 & 0.487 & $N(0,1/ 30 )$ & 0.0533 & 0.0146 & 100 & 0.059 & 0.0158 \\
 & $[10^5]$ & 0.1120 & 0.4885 & $N(0,1/ 30 )$ & 0.0537 & 0.0147 & 100 & 0.0592 & 0.0157 \\
 & $[30^2]$ & 0.1198 & 0.4970 & $N(0,1/ 50 )$ & 0.0515 & 0.0157 & 100 & 0.0601 & 0.0172 \\
 & $[100^2]$ & 0.1623 & 0.5317 & $N(0,1/ 90 )$ & 0.0546 & 0.0184 & 100 & 0.0588 & 0.0188\\
 & $[300^2]$ & 0.3538 & 0.8122 & $N(0,1/ 150 )$ & 0.0499 & 0.0206 & 100 & 0.0617 & 0.023 \\
 & $[800^3]$ & 4.2854 & 2.4$\times10^{13}$ & $N(0,1/ 210 )$ & 0.0794 & 0.0762 & 100 & 0.0918 & 0.0847  \\
\midrule
\multirow{6}{*}{NIG}
 & $[10^2]$ & 0.0308 & 0.4061 & $N(0,1/ 210 )$ & 0.0042 & 0.0002 & 100 & 0.022 & 0.0018 \\
 & $[10^5]$ & 0.0313 & 0.4064 & $N(0,1/ 210 )$ & 0.0042 & 0.0002 & 100 & 0.0216 & 0.0018  \\
 & $[30^2]$ & 0.0366 & 0.4111 & $N(0,1/ 210 )$ & 0.0054 & 0.0004 & 100 & 0.0208 & 0.002 \\
 & $[100^2]$ & 0.0703 & 0.4402 & $N(0,1/ 210 )$ & 0.0095 & 0.0015 & 100 & 0.0225 & 0.0032 \\
 & $[300^2]$ & 0.2778 & 0.7767 & $N(0,1/ 210 )$ & 0.0183 & 0.0059 & 100 & 0.0302 & 0.0077 \\
 & $[800^3]$ & 4.9759 & 1.92$\times 10^{16}$ & $N(0,1/ 210 )$ & 0.0774 & 0.074 & 100 & 0.0816 & 0.0769 \\
\midrule
\multirow{6}{*}{2mT}
&  $[10^2]$ & 0.8268 & 1.327 & $N(0,1/ 10 )$ & 0.9144 & 1.4499 & 40 & 0.6835 & 1.2444\\
& $[10^5]$ & 0.8462 & 1.3462 & $N(0,1/ 10 )$ & 0.9477 & 1.3918 & 40 & 0.6982 & 1.1753\\
& $[30^2]$ & 1.0783 & 1.5811 & $N(0,1/ 10 )$ & 0.6988 & 2.5977 & 40 & 0.6029 & 2.3616\\
& $[100^2]$ & 1.0965 & 1.661 & $N(0,1/ 30 )$ & 0.7839 & 1.8379 & 40 & 0.6279 & 1.5409 \\
 & $[300^2]$ & 2.0737 & 12.0285 & $N(0,1/ 50 )$ & 0.6259 & 2.0289 & 40 & 0.6214 & 2.1892\\
& $[800^3]$ & 8.9462 & 3.25$\times 10^{16}$ & $N(0,1/ 130 )$ & 0.6298 & 2.1513 & 40 & 0.6343 & 2.5071\\
\bottomrule
\end{tabular}}
\caption{Performance of MMAF-guided learning in training, validation, and test for the GAU, NIG, and 2mT data sets. The value of the target function and the PAC Bayesian bound in training are averaged across the spatial positions considered in the study; the CRPS is averaged across time horizons and spatial positions; the MSE is averaged across time horizons.}
\label{tab:workflow}
\end{table}


\subsection{Comparison of model results}
 
We compare MMAF-guided learning applied to SFFNs with the baseline ConvLSTM, ConvGRU, and DiffSTG models described in Section \ref{baselines}. We also consider an alternative version of DiffSTG, namely \emph{DiffSTG embedded}, which uses, in its training, the set of features selected by the Algorithm \ref{algorithm1}. With this, we aim to determine whether the feature-extraction methodology we introduce can improve the performance of a diffusion model. In fact, it is important to note that different datasets are used for training ConvLSTM, ConvGRU, and DiffSTG models compared to the SFNNs used in the study. The reason is that for the former models, we do not use the optimization routine described in Algorithm \ref{algorithm2}. Such models use the entire raster dataset for training. For the GAU and NIG, this corresponds to $m=999192$, and for the 2mT data set, to $m=87672$. The training data set used in MMAF-guided learning is then significantly smaller than that used for the baseline models. Regarding the time horizons considered in our forecasting tasks, we use those corresponding to the test data set for an SFNN described in Section \ref{emb}, to maintain consistency in the forecast evaluation. 

\begin{figure}[!htbp]
\centering
\includegraphics[scale=0.4]{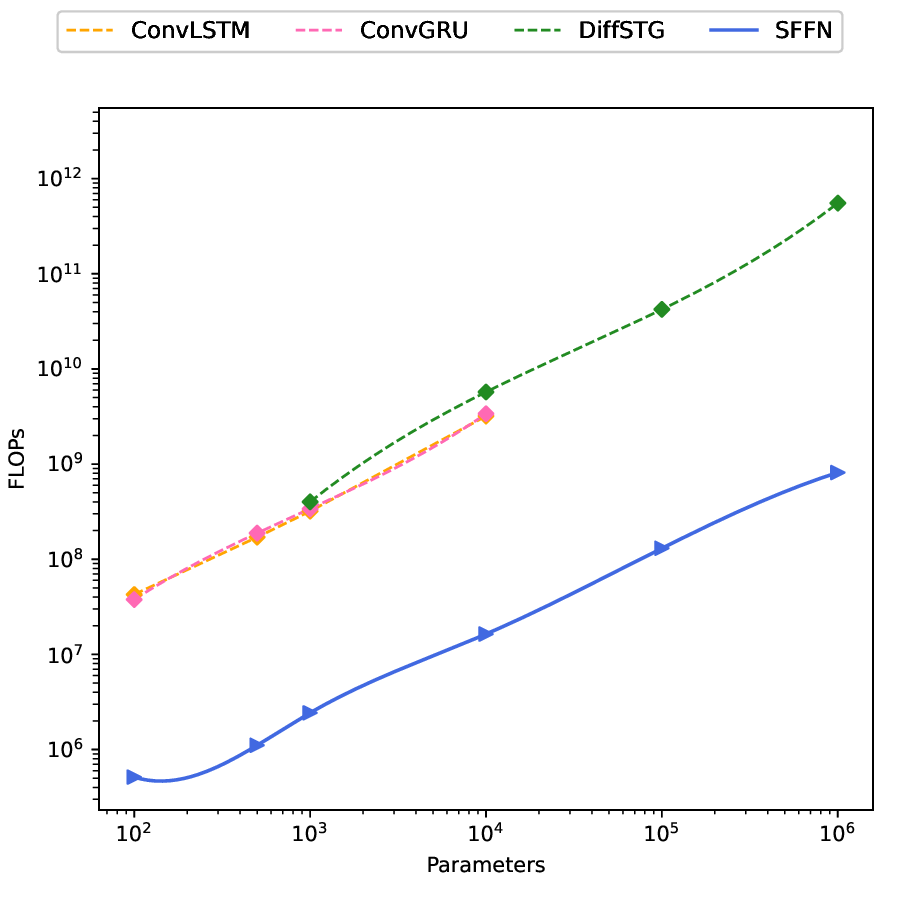}
\caption{\footnotesize Comparison of the computational cost in FLOPs of a training iteration for an ensemble of $8$ SFNNs against the baseline models in the case of the 2mT data set.}
\label{comparison:flops}
\end{figure}

Figure \ref{comparison:flops} shows the number of FLOPs, i.e., floating point operations per second, used during a training iteration for the stochastic feed-forward neural networks and the baseline models used in the case of the 2mT dataset. Figure \ref{comparison:flops} shows that the training phase in MMAF-guided learning is more computationally efficient than in the other setups. The same pattern is observed across different architectures also in the case of the GAU and NIG data sets. 

\begin{figure}[!htbp]
    \includegraphics[width=\columnwidth]{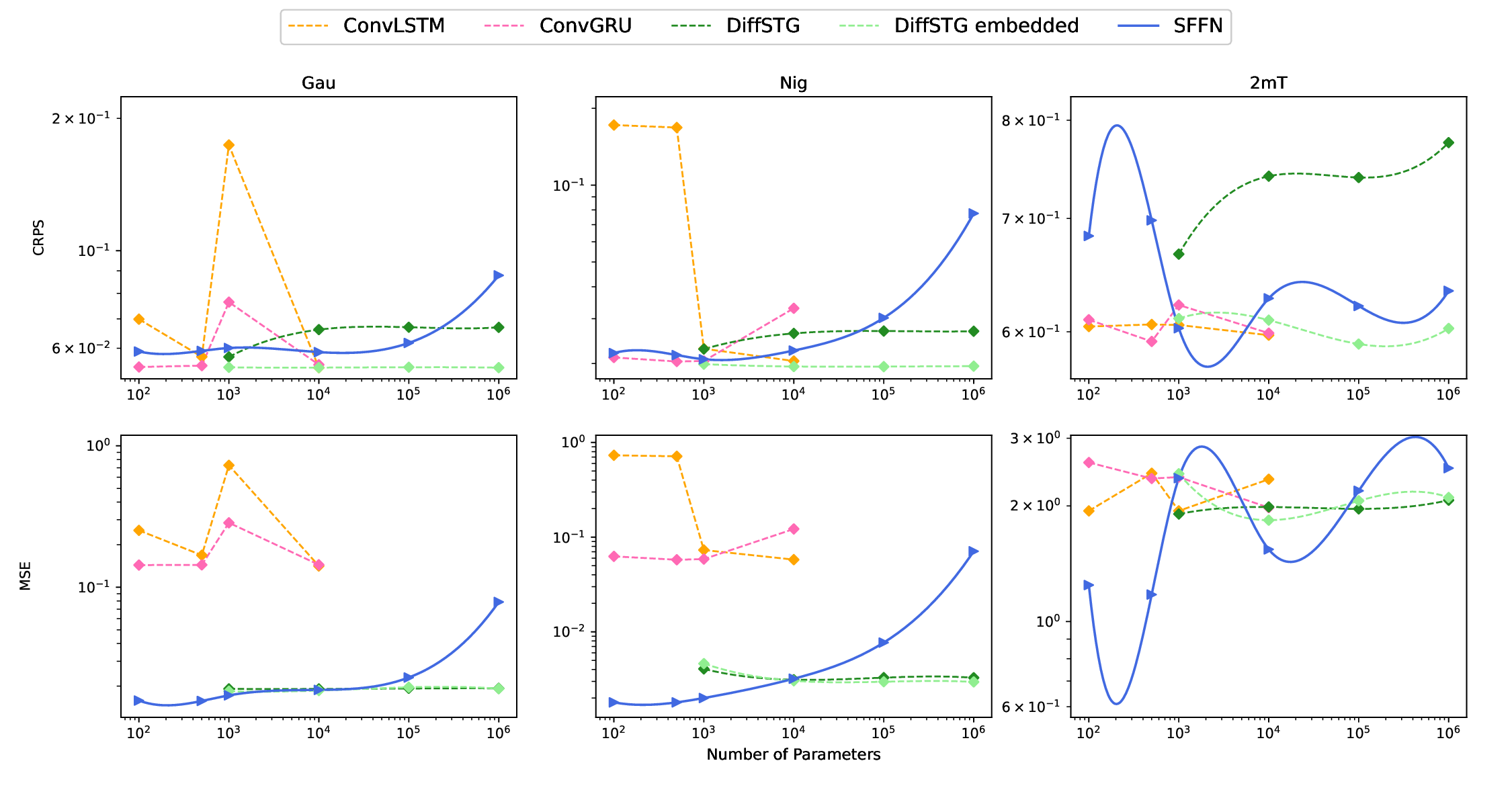}
    \caption{Comparison between the ensemble of $8$ SFNNs and the baseline models on the test set for the datasets GAU, NIG, and 2mT with respect to the CRPS, the MSE and the number of model parameters.}
    \label{comparison}
\end{figure}

We then analyze the CRPS and MSE score of the ensemble forecasts obtained for all architectures described in Table \ref{tab:model_parameters} across the datasets GAU, NIG, and 2mT using the CRPS and MSE and report such results in Figure \ref{comparison}. We are mostly interested in assessing whether our ensemble forecasts are calibrated over time, so we select a list of models with the lowest CRPS for our comparative analysis later in the section (see Table \ref{tab:bestofthebest}). For these models, we also report the ensemble forecast MSE over time. When evaluating model performance using MSE, it is important to note that it measures the forecast's bias and variance relative to the test set, providing an assessment of how forecasts degrade over time. In general, shallow networks achieve better performance than deeper ones, which typically overparameterize and therefore exhibit higher variance (see Figure \ref{comparison}), except in the diffusion framework.

\begin{table}[h!]
\centering
\resizebox{\textwidth}{!}{
\begin{tabular}{|c |ccc|ccc|ccc|}
\toprule
\multirow{2}{*}{Model} & \multicolumn{3}{c|}{GAU}  & \multicolumn{3}{c|}{NIG} &  \multicolumn{3}{c|}{2mT} \\
& Architecture & CRPS & MSE & Architecture & CRPS & MSE & Architecture & CRPS & MSE \\
\midrule
SFFN  &$[100^2]$ & 0.0588 & 0.0188 & $[30^2]$ & 0.0208 & 0.0020 & $[30^2]$ & 0.6029 & 1.5332 \\
&&&&&&&&&\\
\multirow{2}{*}{ConvLSTM} & 150($3\times1$)- & \multirow{2}{*}{0.0542} & \multirow{2}{*}{0.0200}
 & 150($3\times1$)- & \multirow{2}{*}{0.0205} & \multirow{2}{*}{0.05791}
 & 2($9\times1$)- & \multirow{2}{*}{0.5879} & \multirow{2}{*}{1.4515} \\
 & 1($3\times 3$) & & & 1($3\times 3$) & & & 1($9\times 9$) & & \\
&&&&&&&&&\\
\multirow{2}{*}{ConvGRU} & 2($3\times1$)- & \multirow{2}{*}{0.0544} & \multirow{2}{*}{0.0206}
 & 10($3\times1$)- & \multirow{2}{*}{0.0204} & \multirow{2}{*}{0.0033}
 & 18($9\times1$)- & \multirow{2}{*}{0.5921 } & \multirow{2}{*}{1.5319} \\
 & 1($3\times 3$) & & & 1($3\times 3$) & & & 1($9\times 9$) & & \\
 &&&&&&&&&\\
DiffSTG &(1,1)-1 & 0.0573 & 0.0191 & (1,1)-1 & 0.0228 & 0.0041  & (1,1)-1 & 0.6668 & 1.9044\\
&&&&&&&&&\\
DiffSTG emb. & (1,1)-40 & 0.0542 & 0.0192 & (1,1)-11 & 0.0195 & 0.003 & (1,1)-11 & 0.59 & 2.0616\\
\bottomrule
\end{tabular}
}

\caption{: Models with lowest CRPS between the architecture described in Table \ref{tab:model_parameters}.}
\label{tab:bestofthebest}
\end{table}

\begin{table}[h!]
\centering
\resizebox{\textwidth}{!}{
\begin{tabular}{|cc | c |c c|c c|c c| cc |}
\toprule 
&&& \multicolumn{2}{c}{SFFN vs. LSTM} & \multicolumn{2}{c}{SFFN vs. GRU} & \multicolumn{2}{c|}{SFFN vs. DiffSTG} & \multicolumn{2}{c|}{SFFN vs. DiffSTG emb.}\\
&&& Confidence interval & DM p-value & Confidence interval & DM p-value & Confidence interval & DM p-value & Confidence interval & DM p-value \\
\hline
\multirow{18}{*}{\rotatebox[origin=c]{90}{dataset}}& \multirow{6}{*}{\rotatebox[origin=c]{90}{GAU}} & \multirow{3}{*}{\rotatebox[origin=c]{90}{CRPS}}&&&&&&&&\\
&&&$[ 0.0029  ,  0.0067 ]$  & 0.99 & $ [0.0027 ,  0.0069 ]$ & 0.99 & $[ -0.01254 ,  -0.0033 ]$ &  0.0012 & $[ 0.0038  ,  0.0060 ]$ & 0.9999\\
&&&&&&&&&&\\
& & \multirow{3}{*}{\rotatebox[origin=c]{90}{MSE}} &&&&&&&& \\
&&&$[ -0.0019  ,  -0.0004 ]$ & 0.0006 & $[ -0.0025  ,  -0.0009 ]$  & $1.7\times 10^{-6}$ & $[ -0.0017  ,  0.0007]$ & 0.2987 & $[ -0.0000  ,  0.0010 ]$ & 0.9137\\
&&&&&&&&&&\\
\hline
& \multirow{6}{*}{\rotatebox[origin=c]{90}{NIG}} & \multirow{3}{*}{\rotatebox[origin=c]{90}{CRPS}}&&&&&&&&\\
&&&$[ -0.0004  ,  0.0009 ]$ & 0.7672 & $[ -0.0003  ,  0.0010 ]$ & 0.8484 & $[ -0.0102  ,  -0.0030]$ & 0.001 & $[ 0.0002  ,  0.0015 ]$ & 0.9913\\
&&&&&&&&&&\\
& & \multirow{3}{*}{\rotatebox[origin=c]{90}{MSE}} &&&&&&&&  \\
&&&$[ -0.0013  ,  -0.0013 ]$ & $1\times 10^{-16}$ & $[ -0.0013 ,  -0.0013 ]$ & $1\times 10^{-16}$ & $[ -0.0019 ,  -0.0007 ]$ & $3.39\times 10^{-5}$ & $[ -0.0026  ,  -0.0025 ]$ & $1\times 10^{-16}$ \\
&&&&&&&&&&\\
\hline
& \multirow{6}{*}{\rotatebox[origin=c]{90}{2mT}} & \multirow{3}{*}{\rotatebox[origin=c]{90}{CRPS}} &&&&&&&&\\
&&& $[ 0.0049  ,  0.0439 ]$ & 0.99& $[ 0.0229  ,  0.0567 ]$ &  0.99 & $[ -0.0576  ,  -0.0176 ]$ & 0.1389 & $[ 0.0132  ,  0.0459 ]$ & 0.99\\
&&&&&&&&&&\\
&& \multirow{3}{*}{\rotatebox[origin=c]{90}{MSE}} &&&&&&&&\\
&&& $[ 0.2272  ,  0.4313 ]$ &  0.99 & $[ -0.3041  ,  -0.1293 ]$ & 0.0247 & $[ 0.2623  ,  0.4967 ]$ & 0.99 & $[ -0.0286  ,  0.2362 ]$ & 0.99\\
&&&&&&&&&&\\
\hline
\end{tabular}
}
\caption{: Comparisons of the ensemble of SFFNs against the baseline models. For each dataset, the best models from Table \ref{tab:bestofthebest} are compared using $\Delta_{MSE}$ and $\Delta_{CRPS}$. For such quantities, a $95\%$ confidence interval and the p-value of the unilateral Diebold-Mariano test are shown.}
\label{tab:testofthebest}
\end{table}

In our comparative study, we assess the statistical significance of the models' forecasting performance in Table \ref{tab:bestofthebest}.
To this aim, we make use of a standard moving block bootstrap algorithm \citep[Chapter~2]{bootstrap} to compute a $95\%$ confidence interval for the difference of both the average MSE and CRPS across time horizons, between MMAF-guided learning and each baseline model: that is,
\begin{align*}
\Delta_{MSE} &= \frac{1}{H}\sum_{h=1}^H MSE_h(h_\theta) - MSE_h(h_B) \\
\Delta_{CRPS}  &= \frac{1}{H}\sum_{h=1}^H \frac{1}{R}\sum_{r=0}^R\widehat{CRPS}^r_h(h_\theta) - \widehat{CRPS}^r_h(h_{B}),
\end{align*}
where $h_B$ represents a baseline model from those in Table \ref{tab:bestofthebest}.
Finally, a unilateral Diebold-Mariano test with bias correction \citep{dmtest2}, over the same quantities $\Delta_{MSE}$ and $\Delta_{CRPS}$ is performed, to assess if there is statistical evidence that using an ensemble of SFNNs results in better performance than the baselines. As a rule of thumb, negative bootstrap intervals, along with p-values less than $0.05$ (corresponding to a test rejection at level $\alpha=5 \%$), indicate that the SFNNs trained with MMAF-guided learning achieve better performance than their baseline counterparts. A bootstrap interval that is positive or contains zero represents that the baseline model under analysis achieves better or comparable performance than the SFNNs, respectively. The results of the comparison analysis are summarized in Table \ref{tab:testofthebest}.

In terms of CRPS score, the SFNNs within the GAU and NIG frameworks achieve lower scores with shallower architectures than the baselines. For the 2mT dataset, we observe that overparameterization also lowers the score for the SFNNs. Considering the best model setups observed during testing, i.e., the models with the lowest CRPS, a comparison of their performance, as reported in Table \ref{tab:testofthebest}, depends on the data set under analysis.
For the GAU data, the convolutional setups and the diffusion model with features determined as in Section \ref{sec3} achieve better performance than the SFNNs, whereas the SFNNs have better forecasting performance than the diffusion model alone. In the case of Gaussian distributed data, it is to be expected that the performance of the ConvLSTM or ConvGRU is better than the one produced by methodologies that make no assumptions on the shape of the predictive distribution, such as DiffSTG or our SFNNs.
For the NIG data, the performance of the convolutional networks begins to deteriorate, and we can conclude that the forecasts have comparable performance to those produced by the ensemble of SFNNs. Again, we have statistical evidence that the SFNNs have better forecasting performance than the DiffSTG setup, which, however, when equipped with the same features used to train the SFNNs (DiffSTG embedded) returns to being the model with the most reliable forecasts.
Finally, for the 2mT data set, we observe that the SFNNs perform better than the diffusion models, whereas the convolutional setups and the DiffSTG-embedded model achieve better performance. The differences in performance across datasets may be due to the number of data available for training or to the actual dependence structure and stationarity of the data. Overall, the DiffSTG model embedded trained on the features extracted in MMAF-guided learning shows better forecasting performance across all datasets compared to the one trained on the entire raster dataset, with stable performance as the number of parameters increases (i.e., for the deepest architecture presented in Table \ref{tab:model_parameters}). This suggests that the feature-extraction routine used in MMAF-guided learning can also be used as a standalone component in the training of diffusion models. 
Focusing now on the MSE, we observe that the ensembles of SFNNs described in Table \ref{tab:bestofthebest} can outperform the convolutional models in the case of the GAU and NIG data sets. The same neural networks achieve comparable performance to the diffusion setups for the GAU and better performance for the NIG. On the other hand, in the 2mT data set, the LSTM and DiffSTG framework seem to prevail in terms of MSE.

\begin{figure}[!htbp]
\centering
\resizebox{.95\textwidth}{!}{%
\begin{tabular}{c}
\begin{subfigure}{\textwidth}
\includegraphics[width=.47\textwidth]{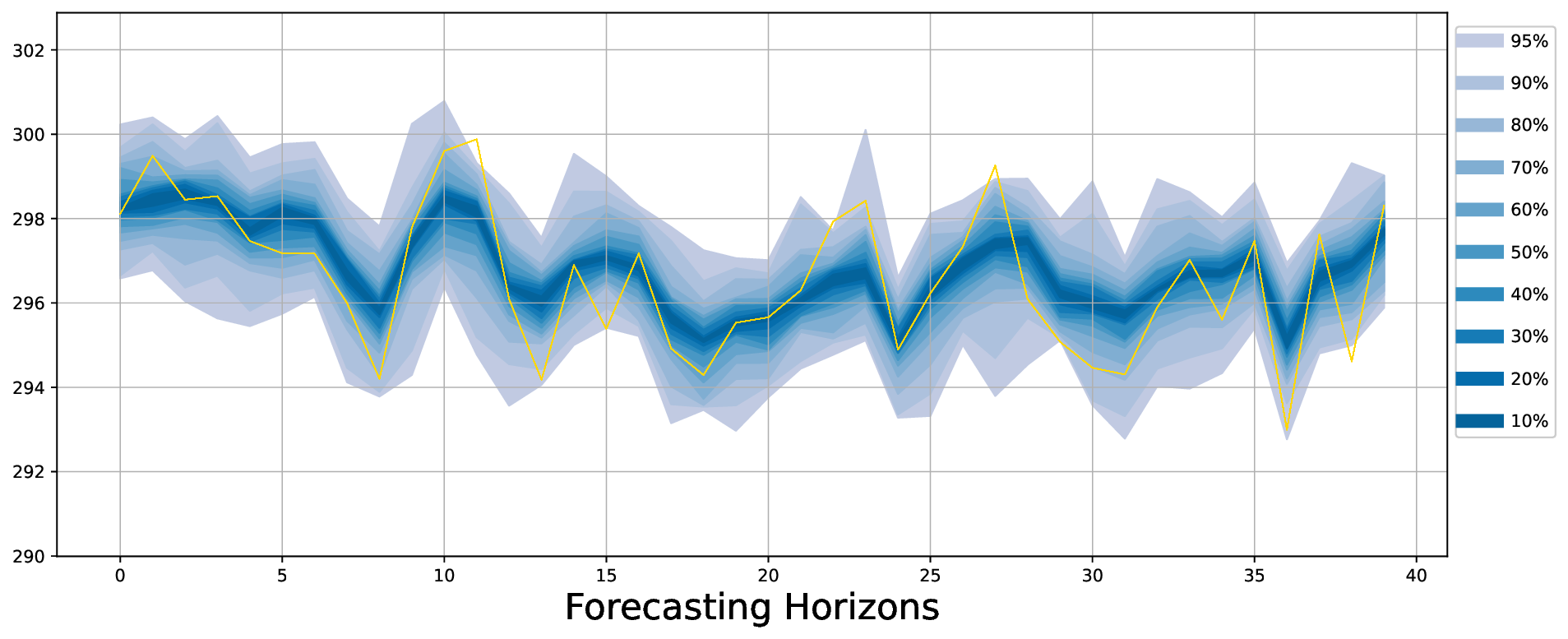}
\includegraphics[width=.47\textwidth]{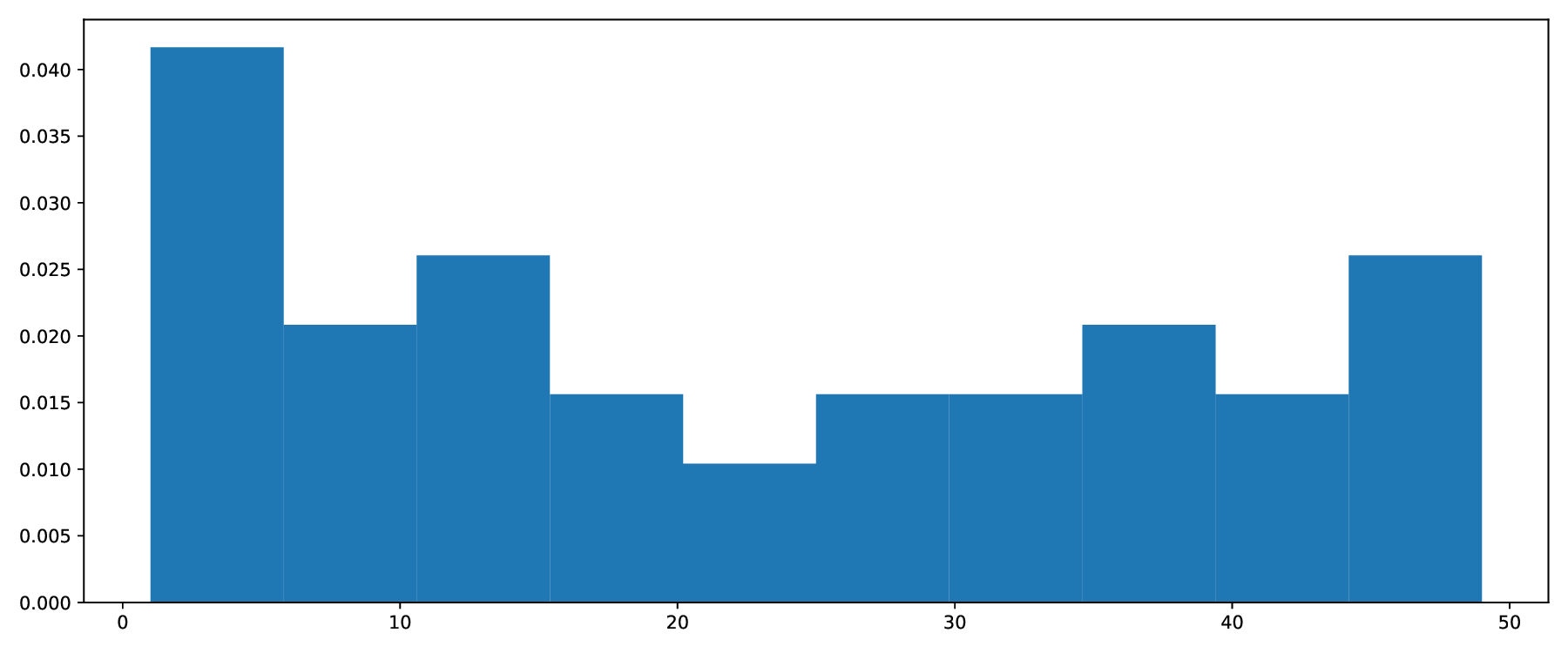}
\caption*{SFFN: $[30^2],\pi\sim N(0,1/10)$}
\end{subfigure}%
\\
\begin{subfigure}{\textwidth}
\includegraphics[width=.47\textwidth]{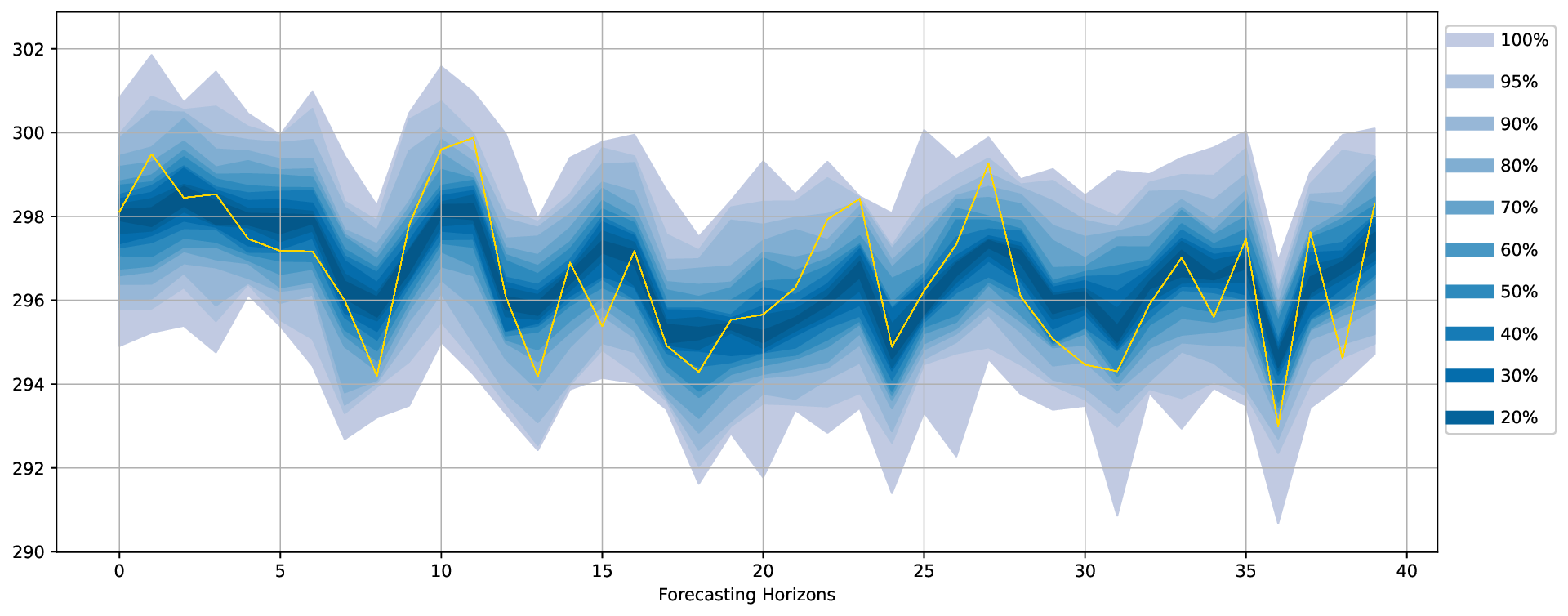}%
\includegraphics[width=.47\textwidth]{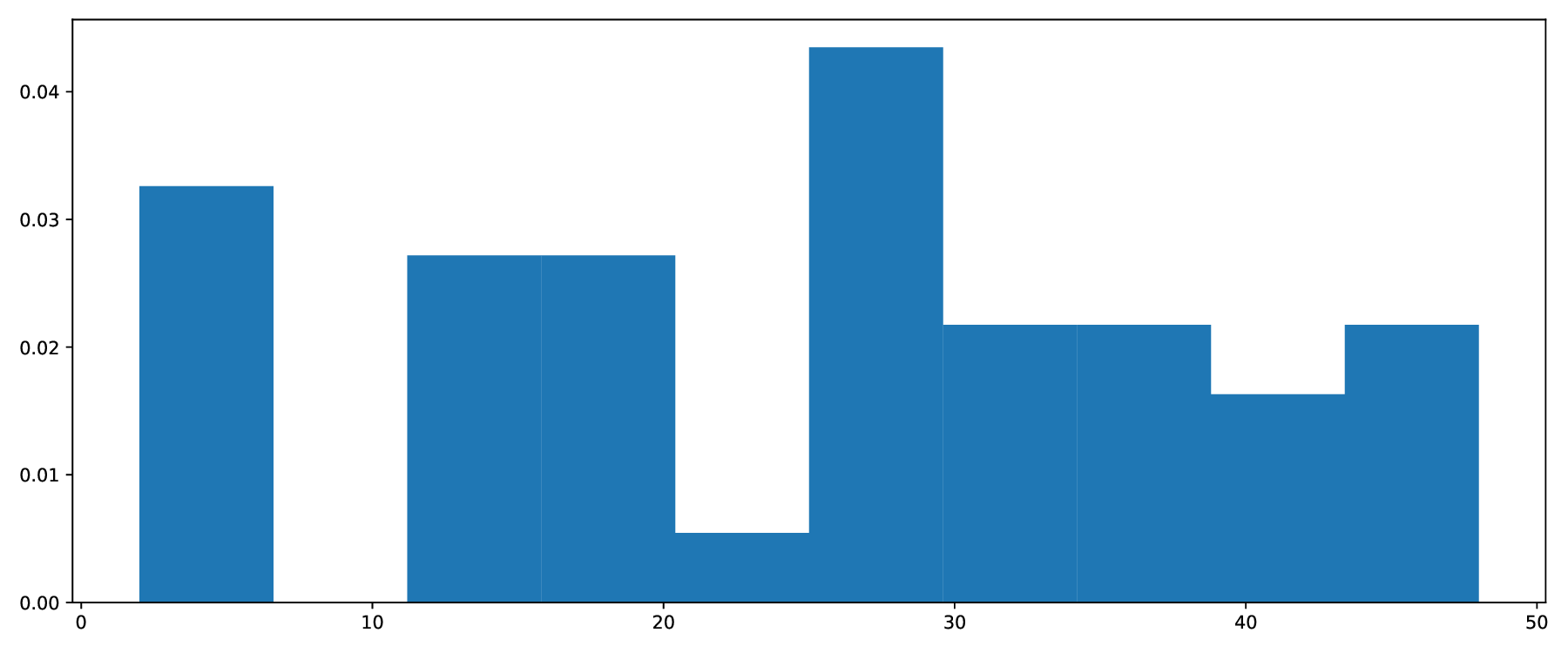}
\caption*{ConvGRU: $18(3\times1)$-$1(3\times3)$}
\end{subfigure}%
\\
\begin{subfigure}{\textwidth}
\includegraphics[width=.47\textwidth]{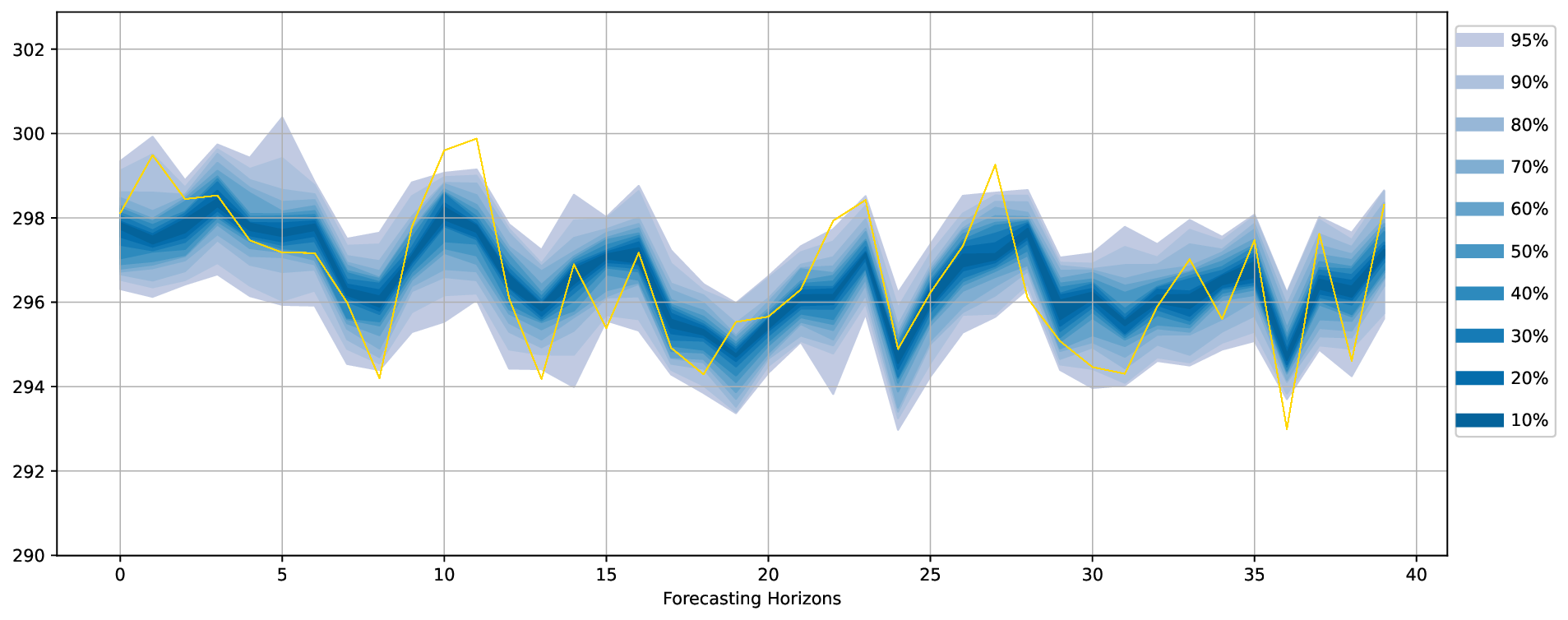}%
\includegraphics[width=.47\textwidth]{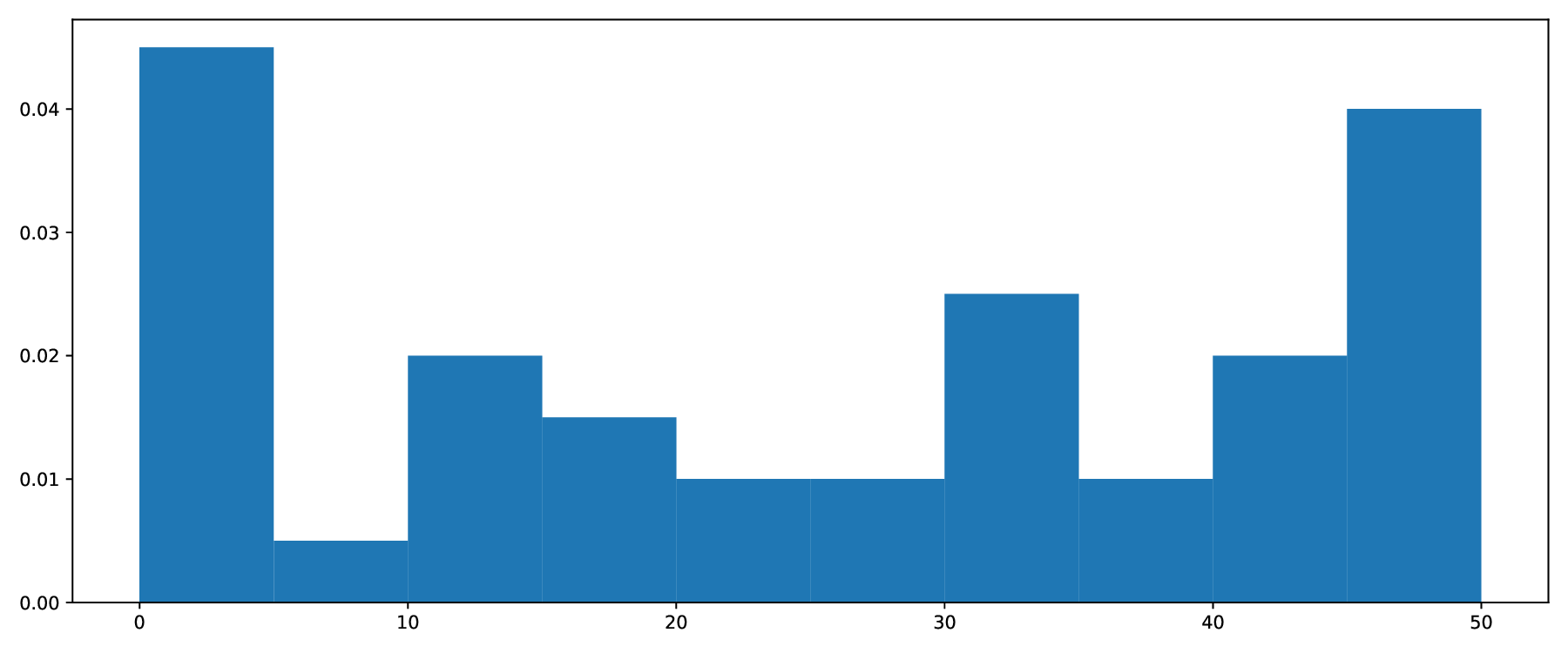}
\caption*{DiffSTG: $(1,1)$-$1$}
\end{subfigure}%
\\
\begin{subfigure}{\textwidth}
\includegraphics[width=.47\textwidth]{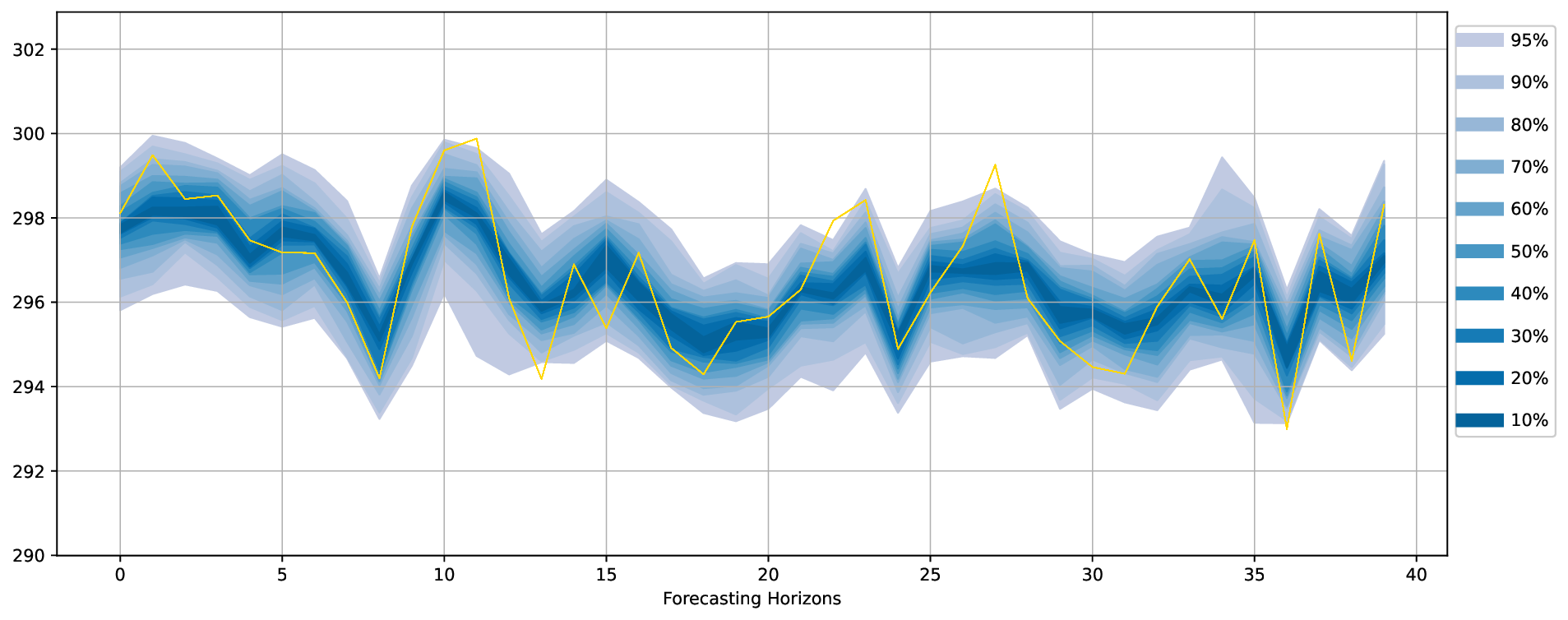}%
\includegraphics[width=.47\textwidth]{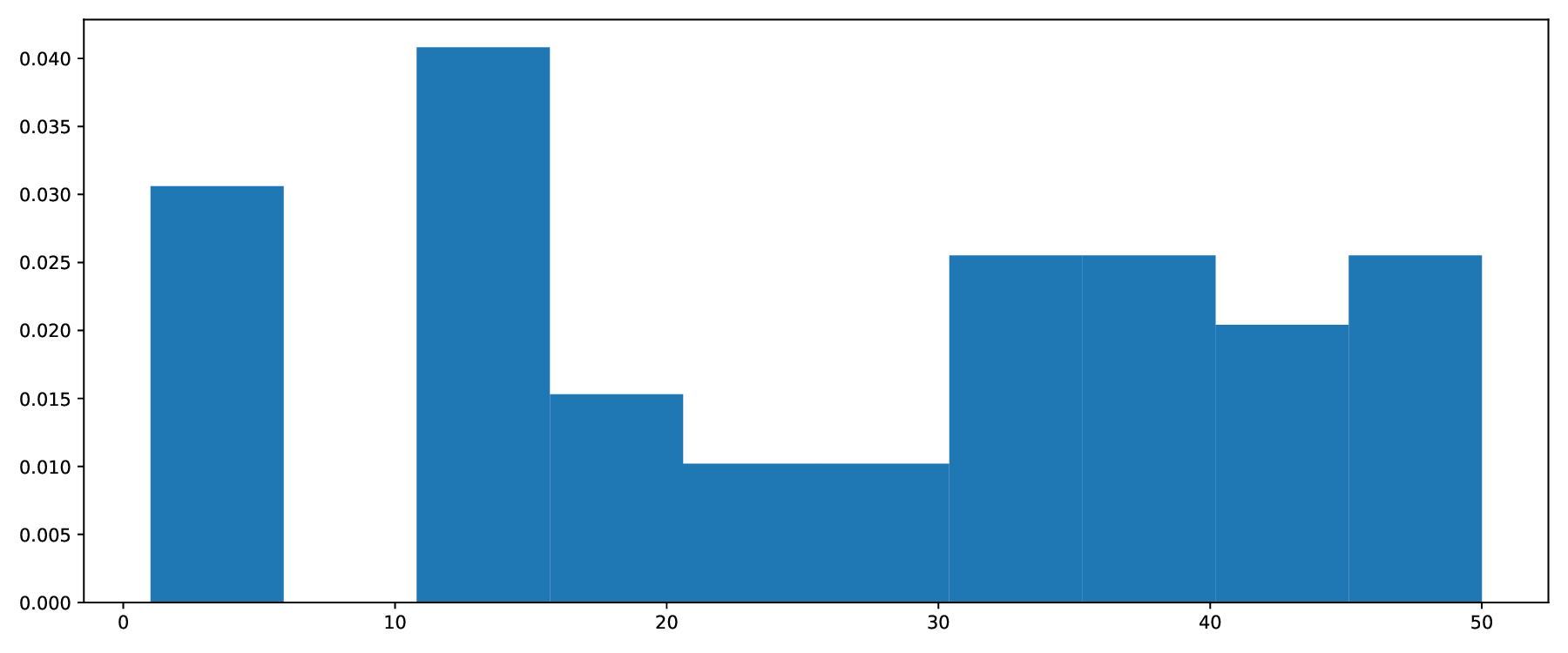}
\caption*{DiffSTG embedded: $(1,1)$-$11$}
\end{subfigure}%
\end{tabular}%
}
\caption{Comparison of ensemble forecasts (left) and their corresponding PIT histogram over time (right), generated using an SFNN, ConvLSTM, and DiffSTG for the $8$th spatial coordinate in our study, respectively, across $40$ horizons versus the 2mT test dataset in gold. The percentages in the legend indicate the quantile ranges of the ensemble forecasts.}
\label{rankhist}
\end{figure}

As last proof of the forecasting quality of MMAF-guided learning, we plot in Figure \ref{rankhist} the prediction intervals (via empirical quantile ranges of the ensemble forecasts) corresponding to a nominal value from $10\%$ to $95\%$ and a pooled PIT histogram across $40$ time horizons for the $8$th spatial position in the case of the 2mT data set and the related models in Table \ref{tab:bestofthebest}. Both are \emph{visual diagnostic tools} to evaluate the calibration of ensemble forecast distributions.
As can be observed, the test data set lies within the prediction intervals and respects their nominal value across setups. The prediction intervals generated by the SFNNs typically follow the trajectory of the test data more closely than in the other setups. The PIT histograms have the best behavior (closeness to a uniform distribution's histogram) for the SFNN. A similar behavior is observed for the prediction interval and the PIT histograms at every spatial position analyzed in the raster, as well as in the GAU and NIG frameworks. 

All in all, the results of this section highlight what has been observed so far when using a \emph{theory-guided machine learning methodology}. Guiding the training routine with information from a data generation process enables the design of methodologies with performance comparable to that of a \emph{pure data-driven generative model} using a smaller amount of data and a more efficient training routine.

\section{Conclusions and future work}
\label{sec5}
We present a novel workflow called \emph{MMAF-guided learning} for feature extraction, training, validation, and inference of stochastic feed-forward neural networks, which produces probabilistic forecasts for a 2D raster dataset across multiple time horizons. Our methodology assumes that an STOU process generates the data. The latter is a special case of a class of fields called MMAF, which in turn belongs to the class of the Ambit fields described in \cite{Ambit}. The latter is a class of random fields that can be employed to model non-stationary phenomena. Our probabilistic forecasting methodology has the potential to be applied to such a class, once the embedding and feature learning from the data are modified to accommodate non-stationary data. The latter is a promising direction for research and will be addressed in future work. Future directions of our work also investigate a PAC Bayesian bound for MMAF-generated data that allows us to consider unbounded loss functions and therefore define a more versatile optimization routine for the training of an ensemble of stochastic feed-forward neural networks.

\bibliographystyle{unsrtnat}
\bibliography{ImmafReviewed}

\begin{appendices}

\section{Estimating the parameters of an STOU process and the decay rate of its $\theta$-lex weakly dependent coefficients}
\label{est}


To enable automatic feature extraction in MMAF-guided learning, we first need to estimate the parameters $c$, which we call \emph{speed of information propagation}, and the parameter $A$ called the \emph{mean-reverting parameter}. Such parameters are estimated using normalized spatial and temporal variograms defined as 
\begin{align*}
	\gamma^S(u)&:= \frac{\E((\bb{Z}_t(x)-\bb{Z}_t(x-u))^2)}{Var(\bb{Z}_t(x))}=2 \Big( 1-\exp\Big(-\frac{Au}{c} \Big)\Big),
\\
	\gamma^T(s)&:= \frac{\E((\bb{Z}_t(x)-\bb{Z}_{t-s}(x))^2)}{Var(\bb{Z}_t(x))}=2(1-\exp(-As)).
\end{align*}
Let $N(u)$ be the set containing all the pairs of indices at spatial distance $u>0$ and with same temporal instance, and $N(s)$ be the set containing the pairs of indices where the observation times are at distance $s>0$ and having same spatial position. 
The empirical normalized spatial and temporal variograms are then defined as follows:
\begin{align*}
	&\bb{\hat{\gamma}}^S(u)=\frac{1}{|N(u)|} \sum_{i,j \in N(u)} \frac{(\bb{Z}_{t_i}(x_i)-\bb{Z}_{t_j}(x_j))^2}{\bb{\hat{k}_2}} 
    \\
	&\bb{\hat{\gamma}}^T(s)=\frac{1}{|N(s)|} \sum_{i,j \in N(s)} \frac{(\bb{Z}_{t_i}(x_i)-\bb{Z}_{t_j}(x_j))^2}{\bb{\hat{k}_2}} 
\end{align*}
where $|N(u)|$ and $|N(s)|$ simply represent the number of the obtained pairs, and $\boldsymbol{\hat{k}_2}$ is the empirical variance
$$
\bb{\hat{k}}_2=\frac{1}{\tilde{D}-1}\sum_{(t_i,x_i)\in \mathbb{T} \times \mathbb{L}} \bb{Z}_{t_i}^2(x_i),
$$
where $\tilde{D}$ denotes the sample size of all space-time observations in the raster dataset. By matching the empirical and the theoretical forms of the normalized variograms, we can estimate $A$ and $c$ by employing the estimators:
\begin{equation}\label{almut1}
	\bb{A^{*}}=-\tau^{-1} \log \Big( 1-\frac{\bb{\hat{\gamma}^T}(\tau)}{2}  \Big), \,\,\, \textrm{and} \,\,\, \bb{c^{*}}=-\frac{\bb{A^*}u}{\log \Big(1-\frac{\bb{\hat{\gamma}^S}(u)}{2}\Big)}.
\end{equation}

Once we have an estimation of the parameters $A$ and $c$, the next step consists in choosing the parameter $a$.
This choice is made using a selection rule related to an estimate of the $\theta$-lex coefficients of the underlying field, see (\ref{choice}). We then need an estimate of such coefficients. For the STOU process, we can compute the $\theta$-lex coefficients explicitly, using Proposition 2.17 in \cite{mmafGL}. Let $\Lambda$ be an $\R$-valued L\'evy basis with characteristic quadruplet $(\gamma, \sigma^2,\nu,\pi)$. If $\int_{|x| >1} x^2 \, \nu(dx) < \infty$, $\gamma+\int_{|x| >1} x \, \nu(dx)=0$, and for $\psi=\psi(r):= \frac{r}{\sqrt{2(c^2+1)}}$, then $\bb{Z}$ is $\theta$-lex weakly dependent with coefficients
	\begin{align*}
		\theta_{lex}(r)&\leq 2 \Big( Var(\Lambda^{\prime}) \int_{A_0(0) \cap (A_0(\psi) \cup A_0(-\psi))} \exp(2As) \, ds\, d\xi \Big)^{\frac{1}{2}} \\
		&\leq 2\Big( 2 Var(\Lambda^{\prime}) \int_{A_0(0) \cap A_0(\psi)}  \exp(2As) \, ds\, d\xi \Big)^{\frac{1}{2}}\\
		&=2\Big( 2 Var(\Lambda^{\prime}) \int_{-\infty}^{-\frac{\psi}{2c}} \int_{\psi+cs}^{-cs} \exp(2As) \, ds\, d\xi \Big)^{\frac{1}{2}}=  \Big( \frac{c}{A^2} Var(\Lambda^{\prime}) \, \exp\Big( \frac{-A \psi}{c} \Big) \Big)^{\frac{1}{2}} \\
		&= 2\Big( \,\frac{c}{A^2} Var(\Lambda^{\prime})  \,\exp\Big( -\,\underset{2\lambda}{\underbrace{\frac{A \min(2,c)}{c}}}\,r \Big) \Big)^{\frac{1}{2}}\\
		&= 2\sqrt{2Cov(\bb{Z}_0(0),\bb{Z}_0(r\min(2,c)))}
	\end{align*}
	
	where $\lambda >0$. 

By estimating the parameter $\lambda$, we then get an estimation of the decay rate of the $\theta$-lex coefficients. This is done using the following plug-in estimator.

\begin{equation}
\label{plug-in}
\bb{\lambda^{*}}=\frac{\bb{A^*}\min(2,\bb{c^*})}{2\bb{c^*}}.
\end{equation}

\section{Dependence Structure of the stochastic process $\mathbf{S}$ }
\label{Inher}
In this section, we prove that the following statement
\begin{lemma}
\label{inerS}
The stochastic process $\bb{S}=(\bb{S}_i)_{i\in\Z}:=(\bb{X}_i,\bb{Y}_i)_{i\in\Z}$ related to the position $x^*$ and with values in $\R^{D+1}$, as defined in Section \ref{emb}, is $\theta$-weakly dependent.
\end{lemma}
 $\theta$- weak dependence is a dependence notion that holds for causal stochastic processes. The relationship between  $\theta$ and $\theta$-lex weak dependence is thoroughly analyzed in \cite{CSS20}. We give the definition of $\theta$-weak dependence for stochastic processes defined on the index set $\Z$ below.

\begin{definition}
    Let $\bb{S}=(\bb{S}_i)_{i\in\Z}$ be a $\R^{D+1}$-valued stochastic process. Then $\bb{S}$ is called $\theta$-weakly dependent if 
    $$
    \theta(k)  = \sup_{u\in \mathbb{N}} \theta_{u}(k) \underset{k\to \infty}{\longrightarrow} 0
    $$
    where 
    \begin{align*}
		\theta_{u}(k) = \sup\left\{ \frac{|Cov(F(\bb{S}_{\Gamma}) , G(\bb{S}_{j})  )|}{\Vert F \rVert_{\infty} Lip(G)} , F\in \mathcal{G}^*_u, \, G\in \mathcal{G}_1, \Gamma, j. \right\}
    \end{align*}
    where $\mathcal{G}^*_u$ represents the set of the bounded functions on $\R^{u(D+1)}$, $\mathcal{G}_1$ the set of bounded and Lipschitz functions on $\R^{(D+1)}$, $\Gamma = \left\{i_1, \dots,i_u \right\} \subset \Z $, $j\in \Z$ such that $i_1\leq i_2 \leq \dots \leq i_u+k =j$. We call $(\theta(k))_{k\in\mathbb{R}^+}$ the $\theta$-coefficients.
\end{definition}

Before giving the proof of Proposition \ref{inerS}, let us remind the structure of the random vector $\bb{S}:=(\bb{S}_i)_{i \in \Z}$. 
\[
\bb{S}_i=(\bb{X}_i,\bb{Y_i})=(\bb{Z}_{i_1}(x_1),\ldots,\bb{Z}_{i_{D}}(x_{D}),\bb{Z}_{t_0+ia}(x^*))
\]
where $(i_s,x_s) \in I(t_0+ia,x^*)$ and $(i_s,x_s)<_{lex} (i_{s+1},x_{s+1})$ for all $s=1,\ldots, D$.

We define now a \emph{truncation for the random vector} $\bb{S}_j$ as follows, let $\psi:=\psi(r)$ for $r \in \R_+$
\[
\bb{S}_j^{(\psi)}=(\bb{Z}_{j_1}^{(\psi)}(x_1),\ldots,\bb{Z}_{j_{D}}{(\psi)}(x_{D}),\bb{Z}_{t_0+ia}{(\psi)}(x^*)).
\]
Let us call $(t,x)$ a generic spatial-time index appearing in the definition of $\bb{S}_j$, then
\[
\bb{Z}_{t}^{(\psi)}(x)=\int_{A_t(x) \cap V^{\psi}_{t}(x)} \exp(-A(t-s)) \, \Lambda(ds,d\xi),
\]
with $V^{\psi}_{t}(x)$ a set such that 
\begin{itemize}
	\item $|V^{\psi}_{t}(x)|\to \infty$ as $r \to \infty$, 
    \item  $\mathbb{E}[|\bb{Z}_{t}(x)-\bb{Z}^{(\psi)}_{t}(x)|] \to 0$ as $r \to \infty$, and
	\item  such that the random vectors $\bb{S}_{l}$ and $\bb{S}_j^{(\psi)}$ are independent for $l<j$ and represent marginals of the field $\bb{Z}$ at distance $r=(j-l)a-p$.
\end{itemize}
A formal definition of the sets $V^{\psi}_{t}(x)$ can be found in \cite{CSS20}, and consists of a truncation of the ambit set in the definition of the random field. The property of independence between random vectors $\bb{S}_{l}$ and $\bb{S}_j^{(\psi)}$ for $l<j$ follows from the property of the L\'evy basis.  

In the following, Lipschitz continuous is understood to mean globally Lipschitz. For $u\in\N$,  $\mathcal{G}_1$ is the class of bounded, Lipschitz continuous functions from $\R^{D+1}$ to $\R$ with respect to the distance $\|\cdot\|_1$ and define the Lipschitz constant as
\begin{gather*}
	\label{Lipcost}
	Lip(h)=\sup_{x\neq y}\frac{|h(x)-h(y)|}{\|x-y \|_1}.
\end{gather*} 

\begin{proof}
For $F\in\mathcal{G}_u^*$, $G\in\mathcal{G}_1$, and for $u \in \N$, $i_1 \leq i_2 \leq \ldots \leq i_u < i_u+k=j$
\begin{align}
    |Cov(F(\bb{S}_{i_1},...,\bb{S}_{i_u}),G(\bb{S}_j))| \leq &|Cov(F(\bb{S}_{i_1},...,\bb{S}_{i_u}),G(\bb{S}_j)-G(\bb{S}^{(\psi)}_j)| \label{B.1}\\
    &+|Cov(F(\bb{S}_{i_1},...,\bb{S}_{i_u}),G(\bb{S}^{(\psi)}_j))| \label{B.2}
\end{align}
We have that \ref{B.2} is equal to zero because of the independence of the random vectors $(\bb{S}_{i_1}, \ldots, \bb{S}_{i_u})$ and $\bb{S}_j$. The term (\ref{B.1}) is less than or equal to
\begin{align}
    & 2 ||F||_{\infty}Lip(G)(\mathbb{E}[\|\bb{S}_j-\bb{S_j^{(\psi)}}\|_1]\\
    & \leq  2 ||F||_{\infty}Lip(G)\underbrace{(D+1)(\mathbb{E}[|\bb{Z}_{t}(x)-\bb{Z}^{(\psi)}_{t}(x)|])}_{\theta(k)}.
\end{align}
Because $r=ka-p$ in relation to the marginals used in (\ref{B.1}), and the properties of the truncation for the random vector $\bb{S}_j$, $\theta(k) \to 0$ for $k \to \infty$, proving that $\bb{S}$ is a $\theta$-weakly dependent process.
\end{proof}

\section{Training algorithm convergence}
\label{convergence}

For each spatial position, we apply Algorithm \ref{algorithm2} and study its convergence with respect to the number of epochs which are setup to be a maximum of $60$ in the GAU and NIG framework, and $5000$ in the 2mT one.

\begin{figure}[H]
\centering
\includegraphics[width=0.7\textwidth]{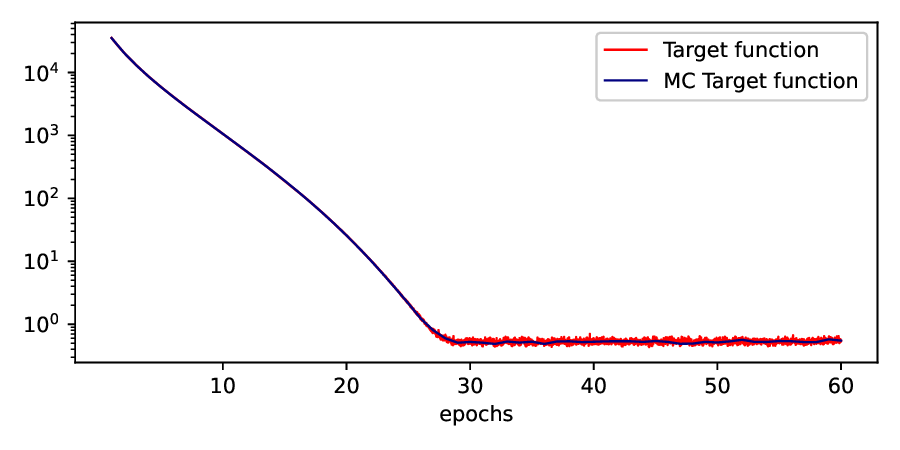}
\caption{Gau, $[300^2]$, $\pi \sim N(0,1/110)$}
\label{fig:convergence1}
\end{figure}

\begin{figure}[H]
\centering
\includegraphics[width=0.7\textwidth]{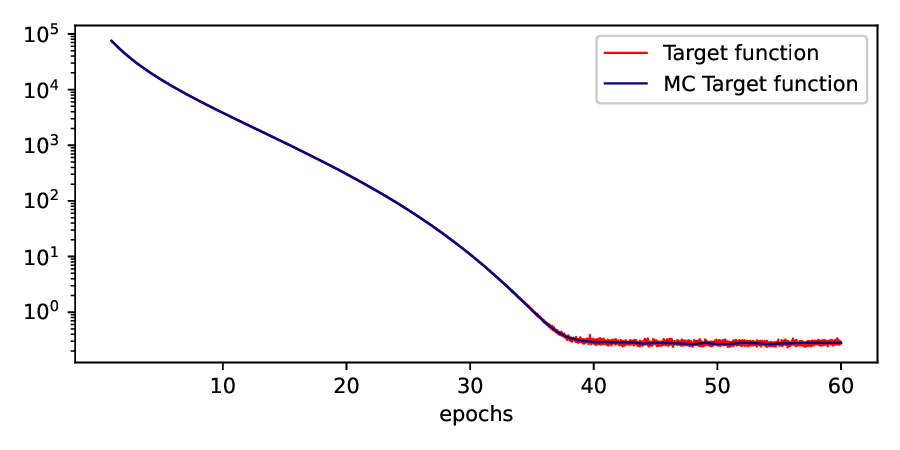}
\caption{NIG, $[300^2]$, $\pi \sim N(0,1/210)$}
\label{fig:convergence2}
\end{figure}

\begin{figure}[H]
\centering
\includegraphics[width=0.7\textwidth]{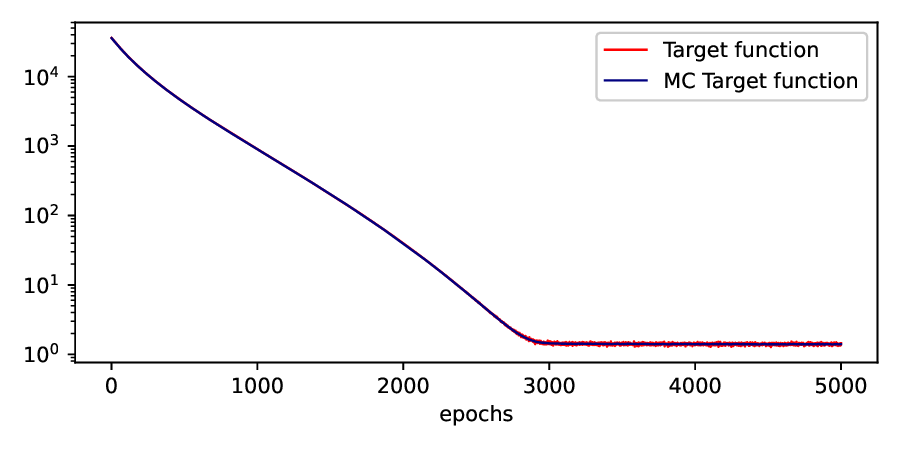}
\caption{2mT, $[300^2]$, $\pi \sim N(0,1/90)$}
\label{fig:convergence3}
\end{figure}

Figures \ref{fig:convergence1},\ref{fig:convergence2}, and \ref{fig:convergence3} show the convergence of the target function for an exemplary stochastic feed-forward neural network used in our empirical study. Note that the curves in Figures \ref{fig:convergence1}, \ref{fig:convergence2}, and \ref{fig:convergence3} report the average values of the target function over the analyzed spatial positions. Our stopping criteria is based on an average evaluation of the value of the epochs at which the plots above tend to zero. 

\end{appendices}


\end{document}